\newif\ifabstract
\abstracttrue
 \abstractfalse 
\newif\iffull
\ifabstract \fullfalse \else \fulltrue \fi

\documentclass[11pt]{article}
\usepackage{comment}
\usepackage{xr-hyper}
\usepackage{hyperref}       
\usepackage{url}            
\usepackage{booktabs}       
\usepackage{amsfonts}       
\usepackage{nicefrac}       
\usepackage{microtype}      
\usepackage{microtype}
\usepackage{graphicx}
\usepackage{subfigure}
\usepackage{booktabs} 
\usepackage{amsfonts}
\usepackage{amssymb,bm}
\usepackage{amstext, enumitem}
\usepackage{amsmath}
\usepackage{xspace}
\usepackage{theorem}
\usepackage{graphicx,color}
\usepackage{url}
\usepackage{graphics}
\usepackage{colordvi, wrapfig}
\usepackage{colordvi}
\usepackage{subfigure}
\usepackage{amssymb}
\usepackage{pifont}

\advance \oddsidemargin by -0.8in
\newcommand{\myparskip}{3pt}
\parskip \myparskip
\newcommand{\sgn}{\text{sgn}}
\newcommand{\requ}{\text{ReQU}}
\newcommand{\loss}{\tilde{L}_{n}(\bm{\bm{\theta}}^{*};p)}
\newcommand{\lossc}{\tilde{L}_{n}(\bm{\bm{\theta}}^{*})}

\newcommand{\e}{\varepsilon}

\newtheorem{theorem}{Theorem}
\newtheorem{lemma}{Lemma}

\newtheorem{corollary}{Corollary}
\newtheorem{claim}{Claim}
\newtheorem{proposition}{Proposition}
\newtheorem{assumption}{Assumption}
\newtheorem{definition}{Definition}

\newenvironment{proof}{\par \smallskip{\bf Proof:}}{\hfill\stopproof}
\def\stopproof{\square}
\def\square{\vbox{\hrule height.2pt\hbox{\vrule width.2pt height5pt \kern5pt
\vrule width.2pt} \hrule height.2pt}}
\newcommand{\footremember}[2]{%
   \footnote{#2}
    \newcounter{#1}
    \setcounter{#1}{\value{footnote}}%
}
\newcommand{\footrecall}[1]{%
    \footnotemark[\value{#1}]%
}

\begin{document}

\title{Revisiting Landscape Analysis in Deep Neural Networks: Eliminating Decreasing Paths
to Infinity}
\author{Anonymous Authors}
\author{Shiyu Liang\footremember{uiuc}{University of Illinois at Urbana-Champaign}\\ sliang26@illinois.edu
          \and Ruoyu Sun\footrecall{uiuc}\\ ruoyus@illinois.edu
            \and R. Srikant\footrecall{uiuc}\\rsrikant@illinois.edu
           }

\date{}

\maketitle

\begin{abstract}
 Traditional landscape analysis of deep neural networks aims to show that no sub-optimal local minima exist in some appropriate sense. From this, one may be tempted to conclude that descent algorithms which escape saddle points will reach a good local minimum. However, basic optimization theory tell us that it is also possible for a descent algorithm to diverge to infinity if there are paths leading to infinity, along which the loss function decreases.
 It is not clear whether for non-linear neural networks there exists
  \textit{one setting} that ``no bad local-min'' and  
  ``no decreasing paths to infinity'' can be simultaneously achieved.
 In this paper, we give the first positive answer to this question.
More specifically, for a large class of over-parameterized deep neural networks with appropriate regularizers, the loss function has no bad local minima and no decreasing paths to infinity. 
The key mathematical trick is to show that the set of regularizers which may be undesirable can be viewed as the image of a Lipschitz continuous mapping from a lower-dimensional Euclidean space to a higher-dimensional Euclidean space, and thus has zero measure.  \end{abstract}


\section{Introduction}

Why are non-convex optimization problems difficult to solve? 
In general, non-convex functions can have many sub-optimal local minima on the landscape, thus local search algorithms can get stuck at one of them.
To distinguish neural network problems from an arbitrary non-convex problem, one goal of landscape analysis has been to show that ``every local minimum is a global minimum''.
For example, this result was proved  for deep linear networks in the reference \cite{kawaguchi2016deep}.

Recently, a number of papers \cite{kawaguchi2016deep,freeman2016topology,choromanska2015loss,nguyen2017loss,nguyen2018loss} 
have studied the problem of showing that all local minima in deep neural networks are good. In particular, \cite{liang2018adding} show that adding a special neuron and a regularizer can ensure that every local minimum is a global minimum.
Intuitively, the extra neuron provides one or more paths in a higher-dimensional space along which the loss function can decrease, starting from each original local minimum. 
However, it has been shown by \cite{kawaguchi2019elimination} that there are simple examples in which the new  path can lead the original local minimum to infinity, and thus a descent algorithm
might diverge to infinity. We remark that, in addition to \cite{liang2018adding}, 
the other works on landscape analysis of neural networks mentioned earlier do not
explicitly eliminate the possibility that a descent algorithm will diverge to infinity. 
 For instance, for a 3-layer 1-dimensional linear network  problem $\min_{ x, y, z \in \mathbb{R} } (xyz - 1)^2$,  although by  \cite{kawaguchi2016deep}, every local minimum is a global minimum, there is a  sequence $(x_k, y_k, z_k) = ( -1/k, \sqrt{k}, 1/k  )  $ diverging to infinity and yet the function values are decreasing and converging to  $1$, which is clearly a sub-optimal value.  This shows that a ``decreasing path'' to infinity exists even for linear neural networks.


Let us step back and rethink why one may want to prove that there exists ``no bad local-min''. 
An underlying reason is that there is much empirical evidence and some theoretical evidence \cite{ge2015escaping, lee2016gradient}  suggesting that gradient descent (GD) type methods can escape saddle points. Together with a ``no bad local-min'' result, one might expect GD type methods to converge to global minima. 
However, an ignored fact is that ``escaping saddle points'' does not
mean GD-type methods will converge\footnote{Note that \cite{ge2015escaping}
do prove the convergence of noisy GD by making a strict saddle assumption which eliminates undesirable functions such as $\exp(-x^2)$.}.
Classical convergence results on GD (e.g. \cite{bertsekas99}) only state ``every limit point of the sequence is a stationary point'', which does NOT ensure that a limit point exists; in other words, does not eliminate the case of divergence. 
 In optimization literature,  explicit assumptions such as ``compact level set'' are often required to ensure convergence (e.g.,  \cite{tseng2001convergence}),
 but for the original neural network problem such an assumption does not hold.

There are two difficulties in non-convex optimization: there could be bad local minima, and decreasing paths to infinity. Therefore, a descent algorithm faces two issues: it may get stuck at a local minimum, or may diverge to infinity. 

\subsection{Our Strategy: New Landscape Analysis}

In this work, we propose to modify the goal from proving 
``no bad local minima''  exist 
to ``\textbf{no bad local minima} and \textbf{no decreasing paths to infinity}'' exist.
On a landscape with ``no bad local minima'', a descent algorithm such as GD can diverge to infinity and  produce a highly sub-optimal loss. 
In contrast, if the new goal is achieved, then a local-search descent algorithm which can escape saddle points will converge to one of the local minima, and since none of these local minima are bad, the point to which convergence occurs will have good properties. In this paper, we define a good local minimum to be one with zero training error.  


How do we ensure there is no decreasing path to infinity? 
A natural idea is to restrict the set of allowable parameters of the neural network to a bounded set. 
However, constraints may introduce additional bad local minima on the boundary of the constraint set. An alternative is to add regularizers, which is what we do in this paper.
However, adding regularizers can also introduce new bad local minima, and we need to properly design the regularizer and the network so that no new bad local minimum are introduced. 
Note that \cite{liang2018adding} also have regularizers, but their regularizers
only control a subset of parameters and allow other parameters to diverge;
we will add regularizers for \textit{every} weight parameter. 

A natural question to ask is:
how do we know that there even exists a neural network  such that it has an associated loss function with the desired landscape? 
Later we will show that for a single-layered neural-net with a $d$-dimensional input and $O(d)$ quadratic neurons, a common loss function with a cubic regularizer has no bad local-min and no decreasing path to infinity

However, neural-nets with linear or quadratic neurons do not have universal representation power,
thus this result is less interesting from a machine learning perspective. 
Therefore, we ask the following question:
\begin{align*}&\textit{   
Does there exist a (deep) neural network with universal representation power}\\
&\quad\quad\textit{ such that an associated loss function has the desired landscape?} 
\end{align*}
In this paper, we prove that the answer is yes, by using both overparameterization and regularization. 
More details are given in the next subsection. 

We remark that overparameterization alone cannot eliminate sub-optimal local minima as counter-examples are given in a recent works \cite{liang2018understanding, ding2019spurious}, but can only eliminate sub-optimal strict local minima \cite{li2018over}.
Therefore, adding regularizers to change the landscape is necessary
if we want to fully eliminate bad local minima.
In other words, both over-parameterization and regularization
help eliminate bad local minima, and the regularization itself
ensures coerciveness. 
The idea of adding regularizer to modify the landscape is indeed critical in the previous work on deep neural networks \cite{liang2018adding}, as well as
works on non-convex matrix completion \cite{keshavan2010matrix,sunluo14}.
However, proving no bad local minima and coerciveness at the same
time  appears to be hard to prove for neural networks, in general. 
So in this paper, we aim to show that indeed such a result holds for  deep convolutional neural-nets with a combination of ReQU neurons (Rectified Quadratic Units)  and Leaky ReLU neurons (Leaky Rectified Linear Units).

%






\subsection{Overview of Main Results} 

In this paper, we consider a binary classification problem where
a neural network classifier is used. 
A  function $F(\bm{\theta}): \mathbb{R}^N \rightarrow \mathbb{R}$ is called coercive if $ \lim_{\| \bm{\theta} \| \rightarrow \infty } F( \bm{\theta} ) = + \infty  $. 
For a coercive function, there is no decreasing path to infinity. 
Our main results are summarized as follows. 
\begin{itemize}[leftmargin=*]
 \item We prove that for a one-hidden-layer neural network with at least
 $ (n + 1) $ ReQU neurons (where $n$ is the number of samples), after adding a proper regularizer, the loss function is coercive and at every local minimum of this loss function, the neural network has zero training error. 
 \item We prove that for an arbitrarily deep convolutional neural network with
 at least $O( d n) $ ReQU neurons in the last hidden layer,
 after adding a proper regularizer, the loss function is coercive and every local minimum of it has zero training error. 
 This is a generalization of the above result. 
\end{itemize}
Our main conclusion is that, for a large class of networks with sufficient representation power, over-parameterization and regularization ensure that every local minimum has zero training error and the landscape does not have a path going to infinity along which the loss function decreases. 

The outline of this paper is as follows. In Section~\ref{sec::prelim}, we present the models and notations used in the rest of the paper. In Section~\ref{sec::main-results}, we present the main results.
Several discussions on the main results are presented in Section~\ref{sec::discussion}. 
The ideas behind the proofs of the main results are presented in Section~\ref{sec::proof} and we conclude the paper in Section~\ref{sec::conclusions}. All proofs not given in the main body
of the paper are given in the appendix.


\subsection{ Related Works}

For deep neural networks, recently there are some  interesting works (e.g. \cite{jacot2018neural,du2018gradient,allen2018convergence,zou2018stochastic,arora2019exact})  that proved that GD converges to global minima at a linear rate
for ultra-wide neural networks.
The conclusion is stronger than landscape anaysis as it can establish linear convergence rate, but they have stronger assumptions (e.g., $O(n^{24})$ neurons in every layer) than ours ($O(n)$ neurons in the last hidden layer). 
That line of research and our landscape analysis take on a different theoretical perspective.
The main question they study is whether polynomial time convergence is possible,
with less emphasis on how many neurons to use -- the focus is on computational complexity. 
Landscape analysis asks whether there are fundamental characteristics of the landscape
that may make a local search algorithm get stuck at an undesirable point, ignoring how much time an algorithm might take to reach a local minimum; in other words, the focus is on geometry.
We view them as two complimentary lines of research, which may converge in the future to provide even stronger results with weaker assumptions.


Our result requires over-parameterization, which is a common assumption in other recent works as well (e.g. \cite{du2018gradient, allen2018convergence,zhou2018convergence,freeman2016topology,nguyen2018loss}). Although over-parameterization can benefit the optimization algorithms (e.g., \cite{arora2018optimization,du2018gradient, allen2018convergence,zhou2018convergence}), it may also increase model complexity, which is partly alleviated by the fact that we add regularizers. In fact, there are other factors that may help control model complexity as well.
Along these lines, many recent works (e.g. \cite{neyshabur2017exploring,bartlett2017spectrally,wei2018margin,neyshabur2018towards}) analyze the generalization error of neural networks.
A particularly interesting topic is to understand why over-parameterization does not cause overfitting. 
One possible explanation is that GD provides ``implict regularization'' that reduces
the generalization error (e.g. \cite{neyshabur2017exploring,li2017algorithmic}). 
Some works provided theoretical evidence that adding more neurons can actually
improve generalization error (e.g., \cite{neyshabur2018towards}). 
This is an interesting line of research that is somewhat orthogonal to our work, but may shed light on why over-parameterization works well in practice despite the possibility of increased model complexity.

There have been many works on the landscape or convergence analysis of shallow neural-nets. 
References \cite{mei2018mean,sirignano2018mean,chizat2018global,rotskoff2018neural} analyzed the limiting behavior of SGD when the number of neurons goes to infinity. 
References \cite{freeman2016topology, soudry2017exponentially,haeffele2017global, ge2017learning, gao2018learning, feizi2017porcupine, panigrahy2017convergence} analyzed the global landscape
 of various shallow networks. 
References \cite{laurent2017multilinear,tian2017analytical,soltanolkotabi2019theoretical,mei2018mean,brutzkus2017globally,zhong2017recovery,li2017convergence,brutzkus2017sgd,wang2018learning,du2018power,oymak2019towards,janzamin2015beating, mondelli2018connection} analyzed gradient descent for shallow networks.


The ReQU neurons and the regularizers used in this paper are carefully designed to achieve the desired landscape. We believe that a careful design may be necessary to achieve a desired landscape, though the scope of networks with desirable landscape
requires further research.

\subsection{Discussions on New Landscape Analysis}

We discuss the relationship between our new landscape analysis and previous works. As mentioned earlier, our main result ensures that there are``no decreasing paths to infinity." 
We note that this does not mean that all previous works suffer from possible decreasing paths to infinity.

In the context of neural networks, the results on algorithmic analysis for neural networks (e.g. \cite{jacot2018neural,allen2018convergence,du2018gradient,zou2018convergence})
 do not have this risk of diverging to infinity under their assumptions. However, the amount of overparameterization that they require is significantly larger than us.
The results on the global landscape (e.g. \cite{kawaguchi2016deep} \cite{liang2018understanding} \cite{liang2018adding}) could have decreasing paths to infinity. However, it is not immediately obvious whether practical algorithms will necessarily diverge to infinity.  
Gradient decent might converge to infinity on such landscapes; however, with clever initialization and other tricks used in practice, whether this will happen is not clear. In contrast, our result ensures that there are no decreasing paths to infinity, thus eliminating a possible issue.

Since the fundamental difficulty in training neural networks is due to the nonconvexity of the objective function, it is also instructive to contrast nonconvexity issues in neural networks with other nonconvex problems. In particular, 
non-convex matrix problems have been studied extensively, but the decreasing paths issue does not seem to appear in any of the models that we are aware of.
There are three classes of results: 
the first class present algorithmic analysis and/or local
geometry and thus do not suffer from the issue considered in this paper (e.g. \cite{keshavan2010matrix,sunluo14,candes2015phase,chen2018gradient}).
The second  class of papers present global geometrical analysis
without using regularizers, but they also prove a convergence result thus
do not suffer from this issue (e.g. \cite{sun2018geometric,bhojanapalli2016global}). 
The third class present global geometrical analysis and do not present
algorithmic proofs, but due to the use of regularizers that lead to coercivity, they do not suffer from this issue (e.g. \cite{ge2016matrix,ge2017no}).

We remark that theoretical results without considering decreasing paths to infinity \cite{kawaguchi2016deep} \cite{liang2018understanding} \cite{liang2018adding}  may still provide useful insight. 
In existing works to date, including this paper, there is a trade-off between the strength of the results
and strength of assumptions.
To eliminate decreasing paths, we restrict the class of neurons and regularizers.
To fully prove algorithmic convergence, the price
paid by \cite{jacot2018neural,allen2018convergence,du2018gradient,zou2018convergence} is the impractical assumptions of a huge number of neurons.
It is an open question whether one can prove convergence with a much smaller amount of overparameterization.

\section{Preliminaries}\label{sec::prelim}

\textbf{Single-layered ReQU network.} Given an input vector of dimension $d$, we consider a single-layered neural network with ReQU activation for binary classification. We denote the activation function of the ReQU neuron by $\text{ReQU}(z)=[\max\{z,0\}]^{2}$. Therefore, the output of a single-layered network consisting of $m$ ReQUs can be expressed by 
$f(x;\bm{\theta})=\sum_{j=1}^{m}a_{j}\requ\left(\bm{w}_{j}^{\top}x+b_{j}\right),$
where the scalar $a_{j}$,  vector $\bm{w}_{j}$ and scalar $b_{j}$ denote the coefficient, weight vector and bias of the $j$-th neuron and $\bm{\theta}$ denotes the vector containing all parameters ($a_j$s, $b_j$s and $w_j$s) in the neural network. 

\textbf{1-D convolutional operator with padding.} Let $\bm{\alpha}=(\bm{\alpha}(1),...,\bm{\alpha}(d_{\alpha}))$ and $\bm{\beta}=(\bm{\beta}(1),...,\bm{\beta}(d_{\beta}))$ denote two vectors of dimensions $d_{\alpha}$ and $d_{\beta}$, respectively. Let $d_{\alpha}<d_{\beta}$. We use $\bm{\alpha}\star\bm{\beta}$ to denote the standard 1-D convolution of vectors $\bm{\alpha}$ and $\bm{\beta}$, where the $j$-th coordinate of the vector $\bm{\alpha}\star\bm{\beta}$ is  
\begin{equation*}
(\bm{\alpha}\star\bm{\beta})(j)=\sum_{i=1}^{d_{\alpha}}\bm{\alpha}(i)\bm{\beta}(i+j-1), \quad j\in[d_{\beta}-d_{\alpha}+1].
\end{equation*}
Let $\bm{\alpha}\ast\bm{\beta}$ denote the 1-D convolution of vectors $\bm{\alpha}$ and $\bm{\beta}$ with padding, where $\bm{\alpha}\ast\bm{\beta}$ is defined as
$$\bm{\alpha}\ast\bm{\beta}=\bm{\alpha}\star(\bm{0}_{d_{\alpha}-1}, \bm{\beta}, \bm{0}_{d_{\alpha}-1}),$$
and $\bm{0}_{d_{\alpha}-1}$ denotes a zero vector of dimension $d_{\alpha}-1$.  Note that $\bm{\alpha}\ast\bm{\beta},$ i.e., convolution with padding, is a vector of dimension $d_\alpha+d_\beta-1.$

\textbf{Multilayer convolutional neural network.} Given an input vector of dimension $d$, we consider a multilayer convolutional neural network $f$ with $l$ layers  of neurons for binary classification. In this paper, we consider a multilayer convolutional neural network of the following architecture: (1) the first  $(l-1)$ layers are convolutional layers where only one channel is used in each layer; (2) all convolutional layers use 1-D convolutional operator with padding; (3) the activation function $\sigma$ for the first $(l-1)$ layers is Leaky ReLU, where $\sigma(z)=z$ if $z\ge0$ and $\sigma(z)=kz$ for some $0<k<1$ if $z<0$;
(3) the last layer is a fully connected layer where the activation function is the rectified quadratic unit (ReQU). Let the $s$-dimensional vector $\bm{v}_{i}$ denote the weight vector for the filter in the $i$-th convolutional layer. Note that all the vectors $\bm{v}_i$ have the same dimension and each convolutional operator uses the same amount of padding. Therefore, the output dimension of each convolutional layer increases by $s-1$ for each layer, where $s$ is the filter length. Let $m$ denote the number of neurons in the $l$-th layer (last layer).
Let the vector $\bm{a}\in\mathbb{R}^{m}$, matrix $\bm{W}\in\mathbb{R}^{(d+(s-1)l)\times m_{}}$ and vector $\bm{b}\in\mathbb{R}^{d+(s-1)l}$ denote the coefficient vector, weight matrix and bias vector of the $l$-th layer. The output of the multilayer convolutional neural network can be written as 
\begin{align}\label{eq::multi-cnn}
f(x;\bm{\theta})=\bm{a}^{\top}\textbf{\requ}\left(\bm{W}^{\top}\bm{\sigma}(\bm{v}_{l-1}\ast \bm{\sigma}(...\ast \bm{\sigma}\left(\bm{v}_{1}\ast x\right))+\bm{b}\right),
\end{align}
where the vector $\bm{\theta}$ contains all parameters $\{\bm{a},\bm{W},\bm{b},\bm{v}_{1},...,\bm{v}_{l-1}\}$ in the network $f$. 

\textbf{Loss and error.} We use $\mathcal{D}=\{(x_{i},y_{i})\}_{i=1}^{n}$ to denote a dataset containing $n$ samples, where $x_{i}\in\mathbb{R}^{d}$ and $y_{i}\in\{-1,1\}$ denote the feature vector and the label of the $i$-th sample, respectively. Given a neural network $f(x;\bm{\theta})$ parameterized by $\bm{\theta}$ and a univariate loss function $\ell:\mathbb{R}\rightarrow\mathbb{R}$, in binary classification tasks, we define the regularized  empirical loss $L_{n}(\bm{\theta};\bm{\lambda})$ as a linear combination of a regularizer $V(\bm{\theta};\bm{\lambda})$ parameterized by a vector $\bm{\lambda}$ and the average loss of the network $f$ on a sample in the dataset. We define the training error (also called the misclassification error) $R_{n}(\bm{\theta};f)$ as the misclassification rate of the network $f$ on the dataset $D$, i.e., 
\begin{equation}
L_{n}(\bm{\theta};\bm{\lambda})=\sum_{i=1}^{n}\ell(-y_{i}f(x_{i};\bm{\theta}))+V(\bm{\theta};\bm{\lambda})
\end{equation}
and
\begin{equation} R_{n}(\bm{\theta};f)=\frac{1}{n}\sum_{i=1}^{n}\mathbb{I}\{y_{i}\neq \sgn(f(x_{i};\bm{\theta}))\},
\end{equation}
where $\mathbb{I}$ is the indicator function. 

\section{Main Results}\label{sec::main-results}

\subsection{Assumptions}\label{sec::assumption}
In this subsection, we introduce two assumptions on the univariate loss function and dataset.

\begin{assumption}[Loss function]\label{assumption::loss}
 Assume that the univariate loss function $\ell:\mathbb{R}\rightarrow\mathbb{R}_{\ge0}$ is non-decreasing and twice differentiable. Assume that there exists $\e>0$ such that any $z\in\mathbb{R}$ satisfying $\ell'(z)<\e$ always satisfies $z<0$.
\end{assumption}
\textbf{Remarks:} (i) Since the univariate loss $\ell$ is non-negative, then the empirical loss $\lossc$ is coercive if the regularizer $V(\bm{\theta};\bm{\lambda})$ is coercive. (ii) We note that all  strictly convex functions satisfy the assumption that there exists a constant $\e$ such that $\ell'(z)<\e$ implies $z<0$, since the derivative of a strictly convex function is strictly increasing. (iii) We provide several examples of some commonly used loss function satisfying the above assumption: (1) the logistic loss, i.e., $\ell(z)=\log_2(1+e^{z})$ and (2) the smooth hinge loss, i.e., $\ell(z)=[\max\{1+z,0\}]^p, p\ge 3$. 

\begin{assumption}[Dataset]\label{assumption::dataset}
Assume that the dataset $\mathcal{D}=\{(x_{i},y_{i})\}_{i=1}^{n}$ satisfies that for any $i,j\in[n]$, $x_{i}\neq x_{j}$ if $i\neq j$.
\end{assumption}
\textbf{Remark:} Assumption~\ref{assumption::dataset} simply states that the feature vector of each sample in the dataset is unique.

\subsection{Single-layered ReQU Network}\label{sec::single}
In this subsection, we present the main result for the single-layered ReQU network.
Recall that for a single-layered ReQU network consisting of $m$ neurons, the output of the neural network is defined as $f(x;\bm{\theta})=\sum_{j=1}^{m}a_{j}\requ\left(\bm{w}_{j}^{\top}x+b_{j}\right)$. Now we define the empirical loss as 

\begin{equation}\label{eq::loss-single}
L_{n}(\bm{\theta};\bm{\lambda})=\sum_{i=1}^{n}\ell(-y_{i}f(x_{i};\bm{\theta}))+\frac{1}{3}\sum_{j=1}^{m}\lambda_{j}\left[|a_{j}|^{3}+2\left(\|\bm{w}_{j}\|^{2}_{2}+b_{j}^{2}\right)^{3/2}\right],
\end{equation}

where all regularizer coefficients $\lambda_{j}$s are positive numbers and the vector $\bm{\lambda}=(\lambda_{1},...,\lambda_{m})$ consists of all regularizer coefficients. We note that after adding the regularizer, the empirical loss $\lossc$ is coercive and always has a  global minimum.  Now we present the following theorem to show that if the network size is larger than the dataset size, i.e., $m\ge n+1$ and the regularizer coefficient vector $\bm{\lambda}$ is carefully chosen, then every local minimum of the empirical loss $\lossc$ achieves zero training error on the dataset $\mathcal{D}$.

\begin{theorem}\label{thm::single}
Let $m\ge n+1$. Under Assumption~\ref{assumption::loss} and \ref{assumption::dataset}, there exists a $\lambda_{0}=\lambda_{0}(\mathcal{D},\ell)>0$ and a zero measure set $\mathcal{C}\subset\mathbb{R}^{m}$ such that for any $\bm{\lambda}\in(0,\lambda_{0})^{m}\setminus \mathcal{C}$, both of the following statements are true:
\begin{itemize}[leftmargin=*]
\item[(1)] the empirical loss $L_{n}(\bm{\theta};\bm{\lambda})$ is coercive.
\item[(2)] every local minimum $\bm{\theta}^{*}$ of the loss $\lossc$ achieves zero training error, i.e., $R_{n}(\bm{\theta^{*}};f)=0$.
\end{itemize}
\end{theorem}
\textbf{Remarks: } (i) Theorem~\ref{thm::single} shows that if the network size $m$ is larger than the dataset size $n$, then there exists a constant $\lambda_{0}$ depending on the dataset $\mathcal{D}$ and loss $\ell$ such that for almost all vectors $\bm{\lambda}$ in the region $(0,\lambda_0)^m$, the empirical loss $\lossc$ is coercive and at every local minimum, the neural network $f$ achieves zero training error. (ii) The assumption that the network size is larger than the dataset size is consistent with the observations in practice where the number of neurons in the widest layer of the mordern networks is comparable to the dataset size~\cite{huang2017densely,he2016deep,krizhevsky2012imagenet}. (iii) In this paper, we use ReQUs at the output layer because ReLUs are nondifferentiable. Further, neurons with the activation function $ \sigma(z)=(z)_+^{(1+\alpha)},\alpha>0$ may also work but may require a different amount of over-parameterization and is a subject of ongoing work. 
(iv) This result is not a pure optimization result as
we only proved that training error $R_n$ is zero.
The training error is the portion of incorrect predictions
which is between $0$ and $1$, while the optimization objective $L_n$
is a non-negative real number that serves as a surrogate function to $R_n$.
We did not prove that every local minimum of $L_n$ is a global minimum
of $L_n$, but a global minimum of the real metric of interest $R_n$.

\subsection{Multilayer Convolutional Neural Network}\label{sec::multi-layer}

Recall that for a multilayer convolutional neural network defined in Eq.~\eqref{eq::multi-cnn} and parameterized by the vector $\bm{\theta}=(\bm{a},\bm{W},\bm{b},\bm{v}_{1},...,\bm{v}_{l-1})$, the output of the neural network is 
\begin{equation}
f(x;\bm{\theta})=\bm{a}^{\top}\textbf{\requ}\left(\bm{W}^{\top}\bm{\sigma}(\bm{v}_{l-1}\ast \bm{\sigma}(...\ast \bm{\sigma}\left(\bm{v}_{1}\ast x\right))+\bm{b}\right).
\end{equation}
Now we define the empirical loss as 
\begin{align}\label{eq::loss-multi}
L_{n}(\bm{\theta};\bm{\lambda})&=\sum_{i=1}^{n}\ell(-y_{i}f(x_{i};\bm{\theta}))+\frac{1}{3}\sum_{j=1}^{m}\lambda_{j}\left[|a_{j}|^{3}+2\left(\|\bm{w}_{j}\|^{2}_{2}+b_{j}^{2}\right)^{3/2}\right]\\
&\quad+\frac{\lambda_{c}}{4}\sum_{k=1}^{l-1}(\|\bm{v}_{k}\|_{2}^{2}-1)^{2},\notag
\end{align}
where all regularizer coefficients $\lambda_{1},...,\lambda_{m},\lambda_{c}$ are positive numbers and the vector $\bm{\lambda}=(\lambda_{1},...,\lambda_{m})$ consists of all regularizer coefficients for parameters in the last layer.  In the following theorem, we will show that if the number of neurons $m$ in the last layer (ReQU layer)  is sufficiently large and the regularizer coefficients $\bm{\lambda}$ and $\lambda_{c}$ are chosen carefully, then every local minimum of the empirical loss $\lossc$ achieves zero training error on the dataset. 
\begin{theorem}\label{thm::multi}
Assume that $m>(d+ls)n$, $l\ge 1$, $s\ge 1$ and $\lambda_{c}>0$. Under Assumption~\ref{assumption::loss} and \ref{assumption::dataset}, there exists a $\lambda_{0}=\lambda_{0}(\mathcal{D},\ell)>0$ and a zero measure set $\mathcal{C}\subset\mathbb{R}^{m}$ such that if $\bm{\lambda}\in(0,\lambda_{0})^{m}\setminus \mathcal{C}$, then both of the following statements are true:
\begin{itemize}[leftmargin=*]
\item[(1)] the empirical loss $L_{n}(\bm{\theta};\bm{\lambda})$ is coercive.
\item[(2)] every local minimum $\bm{\theta}^{*}$ of the loss $\lossc$ achieves zero training error, i.e., $R_{n}(\bm{\theta^{*}};f)=0$.
\end{itemize}
\end{theorem}

The proof of Theorem \ref{thm::multi} will be given in Appendix \ref{appendix::multi}.
Theorem~\ref{thm::multi} shows that if the number of neurons in the ReQU layer $m$, depth $l$, filter length $s$ and dataset size $n$ satisfies that $m>(d+ls)n$, then there exists a constant $\lambda_0$ depending on the dataset and loss  such that  for almost all vectors $\bm{\lambda}$ in the region $(0,\lambda_{0})^{m}$, the empirical loss $\lossc$ is coercive and at every local minimum, the neural network $f$ achieves zero training error. 



\section{Discussions}\label{sec::discussion}
In this section, we discuss the impacts of the network size and activation function on the landscape. 
\subsection{Network Size}
In this subsection, we will show that in Theorem~\ref{thm::single}, the assumption that the network size $m$ is larger than the dataset size $n$ is also a necessary condition. This means that if the network size is smaller than the dataset size, i.e., $m<n$, then we can always find a loss function $\ell$ satisfying Assumption~\ref{assumption::loss} and a dataset $\mathcal{D}$ satisfying Assumption~\ref{assumption::dataset} such that the empirical loss defined by Eq.~\eqref{eq::loss-single} has a bad local minimum.

\begin{proposition}\label{prop::size}
Assume that $m\le n$ and $\ell$ is the logistic loss. Then there exists a dataset $\mathcal{D}$ satisfying Assumption~\ref{assumption::dataset} such that for any $\bm{\lambda}\in(0,1/2)^{m}$, the empirical loss $\lossc$ has a local minimum $\bm{\theta}$ with a training error at least $1-m/n$, i.e., $R(\bm{\theta};f)\ge 1-m/n$.
\end{proposition}

\subsection{Activation Function}
In this subsection, we will show that if the dataset is quadratically separable and the network is a quadratic network, then the required network size can be indepedent of the dataset size. Before presenting the result, we first introduce some additional notation. Given a $d$-dimensional input, we consider a single-layered quadratic network of the form  $f(x;\bm{\theta})=\sum_{j=1}^{m}a_{j}(\bm{w}_{j}^\top x+b_j)^2$, where $m$ denotes the network size and the vector $\bm{\theta}$ contains all parameters $\{a_j,\bm{w}_j,b_j\}_j$ in the network $f$. We consider an empirical loss $\lossc$ of the form given by  Eq.~\eqref{eq::loss-single} and consider a quadratically separable dataset defined as follows. 

\begin{definition}[Quadratically separable]\label{assumption::quadratic}
A dataset $\mathcal{D}=\{(x_{i},y_{i})\}_{i=1}^{n}$ is said to be quadratically separable if there exists a multivariate quadratic function $q:\mathbb{R}^d\rightarrow\mathbb{R}$ such that $y_{i}q(x_{i})>0, i=1,...,n$.
\end{definition}
Now we present the following proposition to show that if the network size is larger than the input dimension, i.e., $m\ge 2d+4$ and the regularizer coefficient vector $\bm{\lambda}$ is chosen carefully, then the network achieves zero training error at every local minimum of the empirical loss.   
\begin{proposition}\label{prop::activation}
Assume $m\ge 2d+4$, $\bm{\lambda}\in\mathbb{R}_{+}^m$  and $\lambda_{i}\neq \lambda_{j}$ for any $i\neq j$.  Under Assumption~\ref{assumption::quadratic}, if the dataset is quadratically separable, then both of the following statements are true:
\begin{itemize}[leftmargin=*]
\item[(1)] the empirical loss $L_{n}(\bm{\theta};\bm{\lambda})$ is coercive.
\item[(2)] every local minimum $\bm{\theta}^{*}$ of the loss $\lossc$ achieves zero training error, i.e., $R_{n}(\bm{\theta^{*}};f)=0$.
\end{itemize}
\end{proposition}
\textbf{Remark: } Proposition~\ref{prop::activation} shows that for a quadratically separable dataset, one can reduce the required network size from $O(n)$ (as required by Theorem~\ref{thm::single}) to $O(d)$, which is independent of the dataset size.
However, the question whether there exists a special activation function for each dataset such that the required network size is independent of the dataset size is still open.

\section{Proof of Theorem ~\ref{thm::single}}\label{sec::proof}


\begin{proof}
The proof of coerciveness is rather straightforward. 
Let the minimal regularization coefficient
be $ \lambda_{\min} \triangleq \min_{j\in[m]} \{ \lambda_j \} > 0 $.
Since the univariate loss $\ell$ is non-negative, then after applying H\"{o}lder's inequality, we can obtain
$
L_n(  \bm{\theta}; \bm{\lambda}) \ge \frac{\lambda_{\min}}{3 \sqrt{2m}}   \| \bm{\theta}  \|_{2}^3.
$
Thus, the empirical loss $\lossc$ is coercive since $L_n(  \bm{\theta}; \bm{\lambda})\rightarrow \infty $  as $\| \bm{\theta} \|_{2} \rightarrow \infty$.

The proof of no bad local-min consists of two steps. \\
\textbf{Step 1}:  for any local minimum there is one neuron that is inactive. \\
\textbf{Step 2}: any local minimum with an inactive neuron
must have zero training error. 

\textbf{Step 1}: We present a lemma, and leave the proof to the next subsection. 
The lemma shows that for almost all $\bm{\lambda}$ in the Euclidean space, at every local minimum of the empirical loss, the neural network always has an inactive neuron. 
Following Lemma~\ref{lemma::single-1} and perturbing the parameters of that  inactive neuron at the local minimum, we finish the proof of Theorem~\ref{thm::single}. 

\begin{lemma}\label{lemma::single-1}
Assume $m\ge n+1$.  There exists a zero measure set $\mathcal{C}(\mathcal{D})\subset\mathbb{R}^{m}$ depending on the dataset $\mathcal{D}$ such that for any $\bm{\lambda}\notin \mathcal{C}$, at every critical point $\bm{\theta}^{*}$ of the empirical loss $L_{n}(\bm{\theta};\bm{\lambda})$, the neural network $f(x;\bm{\theta}^{*})$ always has an inactive neuron, i.e., $\exists j\in[m]$ s.t. $(a^{*}_{j}, \|\bm{w}_{j}^{*}\|_{2},b_{j}^{*})=(0,0,0)$.
\end{lemma}


\textbf{Step 2}: Inactive neuron implies no bad local minimum. 

By Lemma~\ref{lemma::single-1}, there exists a zero measure set $\mathcal{C}$ such that for any $\bm{\lambda}\notin \mathcal{C}$, at every local minimum $\bm{\theta}^{*}$ of the empirical loss $\lossc$, the neural network $f(x;\bm{\theta}^{*})$ has an inactive neuron. Without loss of generality, we assume that  $a_{1}^{*}=0, \|\bm{w}_{1}^{*}\|_{2}=0$ and $b_{1}^{*}=0$. Since $\bm{\theta}^{*}$ is a local minimum of the empirical loss $\lossc$, then there exists a $\delta>0$, such that for any $\tilde{\bm{\theta}}:\|\tilde{\bm{\theta}}-\bm{\theta}^{*}\|<2\delta$, $L(\tilde{\bm{\theta}};\bm{\lambda})\ge L(\bm{\theta}^{*};\bm{\lambda})$. Now we choose the perturbation where we only perturb the parameters of that inactive neuron, i.e., ${a_{1},\bm{w}_{1},b_{1}}$. Let $\tilde{a}_{1}=\delta\sgn(\tilde{a}_{1})$, $(\tilde{\bm{w}}_{1},\tilde{b})=(\delta \bm{u},\delta v)$ for arbitrary $(\bm{u},v):\|\bm{u}\|_{2}^{2}+v^{2}=1$ and $(\tilde{a_{j}},\tilde{\bm{w}}_{j},\tilde{b}_{j})=(a_{j}^{*},\bm{w}^{*}_{j},b_{j}^{*})$ for $j\neq 1$. By the second order Taylor expansion and the definition of the local minimum, we obtain that, for any $\sgn(\tilde{a}_{j})\in\{-1, 1\}$ and any unit  vector $(\bm{u},v):\|\bm{u}\|_{2}^{2}+v^{2}=1$,
\begin{align*}
L(\tilde{\bm{\theta}};\bm{\lambda})&=\loss-\sum_{i=1}^{n}\ell'_{i}y_{i}\delta^{3}\sgn(\tilde{a}_{1})(\bm{u}^{\top}x_{i}+v)_{+}^{2}+R(\delta, \mathcal{D},\bm{u},v)\delta^{6}+\lambda_{1}\delta^{3}\\
&\ge \loss,
\end{align*}
where $R(\delta,\mathcal{D},\bm{u},v)$ is the second order remaining term in the Taylor expansion and $\ell'_i$ is the shorthanded notation for $\ell'(-y_{i}f(x_{i};\bm{\theta}^{*}))$.
This further indicates that for any $(\bm{u},v):\|\bm{u}\|_{2}^{2}+v^{2}=1$, 
$$\left|\sum_{i=1}^{n}\ell'(-y_{i}f(x_{i};\bm{\theta}^{*}))(-y_{i})(\bm{u}^{\top}x_{i}+v)_{+}^{2}\right|<\lambda_{1}<\lambda_{0}$$
From Assumption~\ref{assumption::dataset} and Lemma~8 in Appendix~J, it follows that there exists a single-layered ReQU network $p(x;\bm{\rho})=\sum_{i=1}^{n+1}\alpha_{j}(\bm{\omega}_{j}^{\top}x_{i}+\beta_{j})_{+}^{2}$ of size $n+1,$ with $\bm{\rho}=\{\alpha_j, \omega_j, \beta_j\}_j,$ such that $\min_{i}y_{i}p(x_{i};\bm{\rho})>0$, where the vector $\bm{\rho}$ contains all parameters in the network $p$. Let $$\tilde{\lambda}\triangleq\max_{\bm{\rho}:\|\bm{\rho}\|_{2}=1}\min_{i\in[n]}y_{i}p(x_{i};\bm{\rho}),\quad\bm{\rho}^{*}=\arg\max_{\bm{\rho}:\|\bm{\rho}\|_{2}=1}\min_{i\in[n]}y_{i}p(x_{i};\bm{\rho}),$$ and $\lambda_{0}=\varepsilon\tilde{\lambda}$, then 
\begin{align*}
\tilde{\lambda}\sum_{i=1}^{n}\ell'(-y_{i}f(x_{i};\bm{\theta}^{*}))<\left|\sum_{i=1}^{n}\ell'(-y_{i}f(x_{i};\bm{\theta}^{*}))(-y_{i})p(x_{i};\bm{\rho}^{*})\right|<\lambda_{0}=\e\tilde{\lambda}.
\end{align*}
Since $\ell$ is non-decreasing, then for each $i\in[n]$, $\ell'(-y_{i}f(x_{i};\bm{\theta}^{*}))<\e$. By Assumption~\ref{assumption::loss}, this further indicates that  $y_{i}f(x_{i};\bm{\theta}^{*})>0$ for $\forall i\in[n]$. Therefore, at the local minimum $\bm{\theta}^{*}$, the neural network achieves zero training error, i.e., $R_{n}(\bm{\theta}^{*};f)=0$.
\end{proof}

\subsection{Proof of Lemma~\ref{lemma::single-1}}
\begin{proof}
The proof of  Lemma~\ref{lemma::single-1} is mainly based on
Lemma~\ref{lemma::single-3}, which is the key technical lemma of this paper.
The proof of  Lemma~\ref{lemma::single-3} will be given in Section \ref{subsec: proof of lemma 5.2}, and discussed in detail in Section \ref{sec: journey of proof}. 
\begin{lemma}\label{lemma::single-3}
Let $\mathcal{A}$ denote a set consisting of a finite number of elements. Given a set $\mathcal{X}$ containing $n$ real symmetric matrices $\mathcal{X}=\{X_{i}\in\mathbb{R}^{d\times d}\}_{i=1}^{n}$, a vector $\bm{z}=(z_{1},...,z_{n})\in\mathbb{R}^{n}$, a vector $\bm{\lambda}=(\lambda_1,...,\lambda_m),$  and  a matrix  $\bm{A}=(A_{ij})\in\mathcal{A}^{n\times m}$, define matrices $\bm{M}_{1},...,\bm{M}_{m}$ as
$\bm{M}_{j}(\bm{z};\mathcal{X};\bm{A})=-\sum_{i=1}^{n}z_{i}X_{i}A_{ij}+\lambda_{j} \bm{I}_{d}, j\in[m].$
Given a set $\mathcal{X}$, if $m\ge n+1$, then there exists a zero measure set $\mathcal{C}(\mathcal{X})\subset\mathbb{R}^{m}$ such that for any $\bm{\lambda}\notin \mathcal{C}$, for any $\bm{z}\in\mathbb{R}^{n}$ and for any $\bm{A}\in\mathcal{A}^{n\times m}$, one of the matrices $\bm{M}_{1},...,\bm{M}_{n}$ is non-singular. 
\end{lemma}

We utilize Lemma \ref{lemma::single-3} to prove Lemma~\ref{lemma::single-1}.
Assume that $\bm{\theta}^{*}$ is a critical point of $L_{n}(\bm{\theta};\bm{\lambda})$. Since the empirical loss $\lossc$ is differentiable with respect to $\bm{\theta}$, then by the definition of the critical point, we have 
\begin{align}
\frac{\partial \loss}{\partial a_{j}}&=-\sum_{i=1}^{n}\ell'_{i}y_{i}({\bm{w}_{j}^{*}}^{\top}x_{i}+b^{*}_{j})_{+}^{2}+\lambda_{j}|a^{*}_{j}|a^{*}_{j}=0,\label{eq::lemma-11-1}\\
\nabla_{\bm{w}_{j}}\loss&=-2\sum_{i=1}^{n}\ell'_{i}y_{i}a_{j}^{*}({\bm{w}_{j}^{*}}^{\top}x_{i}+b^{*}_{j})_{+}{{x_{i}}\choose{1}}+2\lambda_{j}\sqrt{\|\bm{w}^{*}_{j}\|_{2}^{2}+{{b}_{j}^{*}}^{2}}{{\bm{w}^{*}_{j}}\choose{b^{*}_{j}}}\label{eq::lemma-11-2}\\
&=\bm{0}_{d+1},\notag
\end{align}
where $\ell'_{i}$ is the shorthand notation for $\ell'(-y_{i}f(x_{i};\bm{\theta}^{*}))$.
Multiplying the both sides of Eq.~\eqref{eq::lemma-11-1} by $a_{j}^{*}$ and taking the inner product of the both sides of Eq.~\eqref{eq::lemma-11-2} with the vector $\left(\begin{matrix}
{\bm{w}_{j}^{*}}^{\top}&b_{j}^{*}
\end{matrix}\right)$, we obtain \begin{align}|a_{j}^{*}|=(\|\bm{w}^{*}_{j}\|_{2}^{2}+{b_{j}^{*}}^{2})^{1/2},\quad \forall j\in[m].\end{align} Therefore, we can only have the following two cases: (1) the network $f(x;\bm{\theta}^{*})$ has an inactive neuron; (2) all neurons of the network $f(x;\bm{\theta}^{*})$ are active. If case (1) holds, then we have proved the lemma. For case (2), since $a_{j}^{*}\neq0$ for all $j\in[m]$, then dividing both sides of Eq.~\eqref{eq::lemma-11-2} by $|a_{j}^{*}|$, we obtain
\begin{align}-\sgn(a_{j}^{*})\sum_{i=1}^{n}\ell'(-y_{i}f(x_{i};\bm{\theta}^{*}))y_{i}({\bm{w}_{j}^{*}}^{\top}x_{i}+b^{*}_{j})_{+}{{x_{i}}\choose{1}}+\lambda_{j}{{\bm{w}^{*}_{j}}\choose{b^{*}_{j}}}=\bm{0}, \quad \forall j\in[m].\end{align}
We can rewrite it as 
\begin{equation}\label{eq::lemma-11-3}
\bm{M}_{j}\left(\begin{matrix}
{\bm{w}_{j}^{*}}^{\top}&b_{j}^{*}
\end{matrix}\right)^{\top}=\bm{0},\quad \forall j\in[m],
\end{equation}
where square matrices $\bm{M}_{1},...,\bm{M}_{m}$ are defined as 
\begin{align*}
\bm{M}_{j}=-\sgn(a_{j}^{*})\sum_{i=1}^{n}\ell'(-y_{i}f(x_{i};\bm{\theta}^{*}))
y_{i}\mathbb{I}\left\{{\bm{w}_{j}^{*}}^{\top}x_{i}+b^{*}_{j}\ge0\right\}{{x_{i}}\choose{1}}\left(\begin{matrix}x_{i}^{\top}& 1\end{matrix}\right)+\lambda_{j}\bm{I}_{d+1}, \end{align*}
 for $\forall j\in[m]$. Now we apply Lemma~\ref{lemma::single-3} to finish the proof. 
Let 
 $z_{i}=y_{i}\ell'(-y_{i}f(x_{i};\bm{\theta}^{*}))$, $\quad A_{ij}=\sgn(a_{j}^{*})\mathbb{I}\{{\bm{w}_{j}^{*}}^{\top}x_{i}+b_{j}^{*}\ge 0\}\quad\text{and}\quad X_{i}={{x_{i}}\choose{1}}\left(\begin{matrix}x_{i}^{\top}& 1\end{matrix}\right).$
Since $\bm{z}=(z_{1},...,z_{n})\in\mathbb{R}^{n}$, $\bm{A}=(A_{ij})_{n\times m}\in\{-1,0, 1\}^{n\times m}$ and all matrices $X_{1},...,X_{n}$ are real symmetric metrices, then by Lemma~\ref{lemma::single-3}, there exists a zero-measure set $\mathcal{C}\subset\mathbb{R}^{m}$ depending on $\mathcal{D}$ such that for any $\bm{\lambda}\notin \mathcal{C}$, one of the matrices $\bm{M}_{j}$ is non-singular.  From Eq.~\eqref{eq::lemma-11-3}, it follows that one of  the vectors $\left(\begin{matrix}
{\bm{w}_{j}^{*}}^{\top}&b_{j}^{*}
\end{matrix}\right)$ is a zero vector and thus the network $f(x;\bm{\theta}^{*})$ has an inactive neuron. 
\end{proof}

\subsection{Proof of Lemma~\ref{lemma::single-3}}\label{subsec: proof of lemma 5.2}
The proof of Lemma \ref{lemma::single-3} is quite non-trivial, though
it only consists of one page.  We will present a formal proof in this subsection,
and later spend a whole Section \ref{sec: journey of proof} to explain the mathematical insights behind this proof. 

This proof consists of three steps.

\textbf{Step 1}:  The key insight is that the set of undesired $\bm{\lambda}$ can be viewed as the image of a mapping from a lower-dimensional Euclidean space to a higher-dimensional Euclidean space, and we want to show the image of this 
mapping has zero measure. 

\textbf{Step 2}: We apply a classical result of Likskii to show
this mapping is Lipschitz continuous.

\textbf{Step 3}: We apply another classical result to show that
the image of this mapping has zero measure, which implies that the set of undesired $\bm{\lambda}$ has zero measure. 


In the following, we first present the two technical lemmas in Step 2 and Step 3,
and finally establish Step 1 and complete the whole proof. 

Let $\mathcal{S}^{d}$ denote the space of $d\times d$ symmetric real matrices. We define a map $\Lambda:\mathcal{S}^{d}\rightarrow\mathbb{R}^{d}$ as the map associating a matrix $\bm{B}\in\mathcal{S}^{d}$ to its eigenvalues in increasing order. A classic result, known as the corollary of Lidskii's Theorem, asserts that
this mapping is Lipschitz continuous.  For a proof, please see the  reference~\cite{kato2013perturbation}. 
\begin{theorem}[Lidskii 1950]\label{thm::lidskii}
 The mapping $\Lambda:\mathcal{S}^{d}\rightarrow \mathbb{R}$ is globally Lipschitz with an explicit constant.
\end{theorem}

 We will need the following property of a globally Lipschitz continuous function.
\begin{proposition}\label{prop: zero measure keeping}[\cite{saks1937theory}]
If a mapping $F:\mathbb{R}^{n}\rightarrow \mathbb{R}^{m}$ is globally Lipschitz continuous and $n<m$, then its image $F(\mathbb{R}^{n})$ has zero measure in $\mathbb{R}^{m}$.  
\end{proposition} 

Finally, we come back to the problem and establsh Step 1 and finish the proof. 
Let the vector $\bm{A}_{j}$ denote the $j$-th column of the matrix $\bm{A}=(A_{ij})_{n\times m}$. Given a set $\mathcal{X}$ and a vector $\bm{A}_{j}$, define a map $\bm{g}:\mathbb{R}^{n}\rightarrow \mathbb{R}^{d}$ as 
$$\bm{g}(\bm{z};\mathcal{X},\bm{A}_{j})=(g_{1}(\bm{z};\mathcal{X},\bm{A}_{j}),...,g_{d}(\bm{z};\mathcal{X},\bm{A}_{j}))=\Lambda\left(\sum_{i=1}^{n}z_{i}X_{i}A_{ij}\right).$$
Given a set $\mathcal{X}$ and a matrix $\bm{A}$, let the set $\mathcal{C}(\mathcal{X},\bm{A})$ denote all possible $\bm{\lambda}\in\mathbb{R}^{m}$ such that all matrices $\bm{M}_{1},...,\bm{M}_{m}$ are singular. This means that for any $\bm{\lambda}=(\lambda_{1},...,\lambda_{m})\in\mathcal{C}(\mathcal{X},\bm{A})$, there exists  a vector $\bm{z}=(z_{1},..., z_{n})\in\mathbb{R}^{n}$ and a series of indices $i_{1},...,i_{m}$ such that $\lambda_{k}$ is the $i_{k}$-th smallest eigenvalue of the matrix $\sum_{i}z_{i}X_{i}A_{ik}$.
Thus, we can write the set $\mathcal{C}(\mathcal{X},\bm{A})$ as follows 
$$\mathcal{C}(\mathcal{X},\bm{A})=\bigcup_{i_{1}\in[n]}\bigcup_{i_{2}\in[n]}...\bigcup_{i_{m}\in[n]}\left\{(g_{i_{1}}(\bm{z};\mathcal{X},\bm{A}_{1}),...,g_{i_{m}}(\bm{z};\mathcal{X},\bm{A}_{m}))|\bm{z}\in\mathbb{R}^{n}\right\}.$$ 
Given a set $\mathcal{X}$ and a matrix $\bm{A}$,  from Lidskii's theorem, it follows that for each $i_{1}\in[n],...,i_{m}\in[n]$, the map $(g_{i_{1}},...,g_{i_{m}}):\mathbb{R}^{n}\rightarrow\mathbb{R}^{m}$ is globally Lipschitz continuous. 
Furthermore, since $m\ge n+1$, then for each $i_{1}\in[n],...,i_{m}\in[n]$,
according to Proposition \ref{prop: zero measure keeping}, 
the set $\left\{(g_{i_{1}}(\bm{z}),...,g_{i_{m}}(\bm{z}))|\bm{z}\in\mathbb{R}^{n}\right\}$ has zero measure in $\mathbb{R}^{m}$. Since the set $\mathcal{C}(\mathcal{X},\bm{A})$ is a union of a finite number of zero measure sets in $\mathbb{R}^{m}$, the set $\mathcal{C}$ has zero measure in $\mathbb{R}^{m}$. Thus, Lemma~\ref{lemma::single-3} follows directly since the set $\mathcal{C}(\mathcal{X})$ consisting of all undesired $\bm{\lambda}$ is a union of a finite number of zero measure set $\mathcal{C}(\mathcal{X})=\bigcup_{\bm{A}\in\mathcal{A}^{n\times m}}\mathcal{C}(\mathcal{X}, \bm{A})$.

\section{Conclusions}\label{sec::conclusions}

 Prior works \cite{liang2018adding, kawaguchi2019elimination} showed that if the goal is just to eliminate all bad local minima, then adding a special neuron  achieves this goal, but this can potentially make a descent algorithm diverge. 
Here, we propose a new type of landscape with no bad local-min and no decreasing paths to infinity. 
To achieve these properties, our strategy is to first add regularizers on all weight parameters to make the loss function coercive, and then show that the new loss function has no local minimum with non-zero training error.
We prove that for an over-parameterized deep CNN (convolutional neural-net) with a combination of ReQU and Leaky ReLU neurons, adding proper regularizers indeed makes the loss function coercive, with no bad local minima. 

\appendix
\section{Understanding Proof of Key Technical Result Lemma \ref{lemma::single-3} }\label{sec: journey of proof}
The proof of Lemma \ref{lemma::single-3} extracts an interesting underlying mathematical problem; for readers who want to get the essence of the math ideas and do not want to follow the somewhat heavy notation of
neural networks, we recommend  directly reading the summary of proof techniques 
in Appendix \ref{subsec::Summary}, and some clean math problems in 
Appendix \ref{subsubsec: counting trick} and Appendix \ref{subsubsec: zero measure preserving}.

The proof can be divided into two main steps.
In Step 1,  Lemma 1  shows that for almost all $\bm{\lambda}$ in the Euclidean space, at every local minimum of the empirical loss, the neural network always has an inactive neuron. 
In Step 2, perturbing the parameters of that  inactive neuron at the local minimum, we show that any local minimum has to
be a global minimum.  This finishes the proof of Theorem 1. 

The main challenge is to prove Lemma 1; and the whole section is devoted
to the proof of Lemma 1. 
We write down the first order condition, which consists of $ m $ equations of a certain matrix $M_j$ times the $j$-th weight vector, 
$j=1, \dots, m$. Thus, to prove one of the weight vectors is zero (thus an inactive neuron), we only need to prove that one of the matrices $M_j$ is full rank.

The expression of $M_j$ is given as (in this section, we ignore the bias term for simplicity)
$$
M_{j}=-\sgn(a_{j}^{*})\sum_{i=1}^{n}\ell'(-y_{i}f(x_{i};\bm{\theta}^{*})) \mathbb{I}\left\{{  \bm{ w }_{j}^{*}}^{\top} x_{i}  \ge0\right\}   y_{i}  x_i x_i^{\top}  +\lambda_{j}\bm{I}_{d},
$$
where $ f(x;\bm{\theta})=\sum_{ k=1}^{m}a_{ k }\requ\left(\bm{w}_{k }^{\top}x \right) \triangleq \sum_{k =1}^m \phi (\theta_k, x)  $,
in which $ \bm{\theta}_k = ( a_k,  \bm{w}_{k} ) $ represents all the weights related to the $ k $-th neuron.
We want to show that we can properly choose $\lambda_1,\dots, \lambda_m $    such that for any weight $\bm{\theta}^* $ one of $M_1, \dots, M_m $ is full rank
\footnote{Here we assume $x_i, y_i$ are fixed a priori. We could 
also study a stronger problem that choosing $\lambda_j$ such that
for any $x_i, y_i$ and any weight $\bm{\theta}^* $, one of the matrices is full rank.
The proof for deep neural network in fact deals with the stronger problem.}.

One difficulty of the problem is that $M_j$ has a quite complicated expression.
We will analyze 1-dimensional case first (already quite nontrivial), and then discuss
how to solve the high-dimensional case. 



\subsection{Solving One-Dimensional Case: Finite Alphabet Trick and Counting Trick}

We first consider 1-dimensional case. In this case, $x_i \in \mathbb{R}$ and 
$$
M_{j}=-\sgn(a_{j}^{*})\sum_{i=1}^{n}\ell'(-y_{i}f(x_{i};\bm{\theta}^{*})) \mathbb{I}\left\{{  \bm{ w }_{j}^{*}}^{\top} x_{i} \ge0\right\}  y_{i} x_i^2  +  \lambda_{j}, \quad  j=1, \dots, m
$$
are scalars, and we want to pick $\lambda_j $ so that for any $\theta_k^*, k =1, \dots, m$,  one of $M_j$ is non-zero.  
Note that $\theta_k^*, k =1, \dots, m$ affect the prediction $f(x_{i};\bm{\theta}^{*})$ and
the indicator variables $ \mathbb{I}\left\{{  \bm{ w }_{j}^{*}}^{\top} x_{i}  \ge0\right\} $. 
This can be viewed as a game: 
the adversary Alice  is allowed to pick  any $ \bm{\theta}^* $
to make all $M_j$ zero, and the defender Bob needs to pick $ \lambda_j, j = 1, \dots, m$ to make one of them non-zero.
In addition, Bob has to pick $ \lambda_j, j = 1, \dots, m$ beforehand. 

To understand the essence of the underlying math problem, we need to simplify the notation.
Note that how to simplify the notation itself reflects the understanding of the problem.
Let
\begin{equation}\label{Xi definition}
 \alpha_i = x_i^2 y_i \ell'(-y_{i}f(x_{i};\bm{\theta}^{*})) 
\end{equation}
to denote terms that is dependent on the sample index $i$ but independent of the neuron choice $j$, and 
\begin{equation}\label{Aij definition}
A_{ij} =  \sgn(a_{j}^{*}) \mathbb{I}\left\{{  \bm{ w }_{j}^{*}}^{\top} x_{i}  \ge0\right\} 
\end{equation}
to denote the terms that depend both on $i$ and $j$. 
Since we assume $x_1,\dots, x_n;$ $y_1, \dots, y_n$ are fixed, 
both $\alpha_i$ and $A_{ij}$ vary according to the weights $\bm{\theta}^*$. 
However, in the proof,  we treat $\alpha_i$ and $A_{ij}$  as ``free'' variables,
thus we do not need to consider the complicated interactions between the weights
$\bm{\theta}^*$ and $\alpha_i$ and $A_{ij}$. 
In other words, we do not need to consider the specific forms of $\alpha_i$ and $A_{ij}$,
and do not need to directly deal with $\bm{\theta}^*$.
\subsubsection{ Underlying Math Problem: First Trial}\label{subsec: math problems}
The problem we want to solve in 1-dim case becomes the following clean mathematical problem.

\textbf{Problem 1}: If $m \geq  n +1 $, then  
there exists $\lambda_j , j=1, \dots, m$ such that for all $A_{ij} \in \mathbb{R},   \alpha_i  \in \mathbb{R}, i=1, \dots, n; j=1, \dots, m  $,
one of $M_j$ has to be non-zero, where
$$
  M_j =  \lambda_j  - \sum_{i=1}^n  A_{ij} \alpha_i  ,  j=1, \dots, m  .
$$

To further understand the problem, let us look at a simpler case that $  n = 1, m = 2 $.

\textbf{Problem 1} ($n=1$ case):  
There exists $\lambda_1, \lambda_2 $ such that for all $ A_{11}, A_{12} \in \mathbb{R},   X  \in \mathbb{R}, j=1, 2  $, 
one of $M_j$ has to be non-zero, where
$$
M_1 =  \lambda_1  -   A_{1 1} \alpha_1  , \quad  M_2 =  \lambda_2  -   A_{1 2} \alpha_1  .
$$

Unfortunately, Problem 1 is not solvable. 
In fact, for any $\lambda_1, \lambda_2 $, we can pick $ \alpha_1 = 1,   A_{1 1} = \lambda_1, A_{1 2} = \lambda_2 $
such that $M_1 = M_2 = 0$.

What is the underlying reason of the failure? 
Let us re-examine Problem 1.
The adversary Alice can choose $A_{ij}, \alpha_i$ after Bob chooses $\lambda_j's$, thus Alice has two advantages:
first, she has more degrees of freedom; second, she acts after Bob. 
How is it even possible for Bob to win? 


\subsubsection{Modified Math Problem: Finite Alphabet Trick}\label{subsubsec: finite alphabet}
The  first  key trick is to eliminate the first advantage of Alice: reduce the degrees of freedom of Alice. 
If we re-examine the definition of $A_{ij}$ in \eqref{Aij definition}, we notice an important property \footnote{This property is due to the fact that we pick the ReQU neurons and thus the derivative depends on $w_j^{\top} x_i$ and an indicator function.}:
$$
  A_{ij} \in \{ 0,  +1,  -1  \}. 
$$ 
The importance is that there are only \textit{finitely many} choices of $A_{ij}$. 

Now we can modify Problem 1 to the following one, which is the true problem we need to solve.
Note that we will consider a stronger requirement ``for almost all $\lambda_j$''
instead of the original requirement ``there exists $\lambda_j$''. 

\textbf{Problem 2}: 
Suppose $\mathcal{A} \subseteq \mathbb{R}$ is a finite set. 
If $m \geq  n +1 $, then  for almost all $\lambda_j , j=1, \dots, m$ the following holds:
for any $A_{ij} \in \mathcal{A},   \alpha_i  \in \mathbb{R}, i=1, \dots, n; j=1, \dots, m  $,
one of $M_j$ has to be non-zero, where
$$
M_j =  \lambda_j  - \sum_{i=1}^n  A_{ij} \alpha_i  ,  j=1, \dots, m.
$$
For our purpose, we have $\mathcal{A} = \{ 0,  +1,  -1  \}$.

To show how to prove Problem 2, we first consider a special problem
(and the cornerstone of the whole proof). 

\textbf{Problem 3} (special case of Problem 2, with fixed $A_{ij}$):
Fix $A_{ij} \in \mathbb{R},  i=1, \dots, n; j=1, \dots, m . $
 If $m \geq  n +1 $, then  for almost all $\lambda_j , j=1, \dots, m$ the following holds:
 for any $ \alpha_i  \in \mathbb{R}, i=1, \dots, n $,
one of $M_j$ has to be non-zero, where
$$
M_j =  \lambda_j  - \sum_{i=1}^n  A_{ij} \alpha_i  ,  j=1, \dots, m.
$$

The proof of Problem 2 is an extension of that for Problem 3. 
The difference is that now we need to consider finitely many choices of $A_{ij}$'s (there are finitely many choices due to the finite alphabet $\mathcal{A}$ that  $A_{ij}$ is drawn from) . 
For each choice of $A_{ij}$'s, the undesirable set of $( \lambda_1 , \dots, \lambda_m) $ has zero-measure.
Thus the union of all undesirable sets also has zero measure. In other words, for almost all $( \lambda_1 , \dots, \lambda_m) $, the desirable property holds. 

\subsubsection{Counting Trick}\label{subsubsec: counting trick}

We note that $M_j$ does not need to appear in the problem, and we rewrite Problem 3
in a cleaner form as follows. 

\textbf{Problem 3b} (equivalent form of Problem 3):
Fix $A_{ij},  i=1, \dots, n; j=1, \dots, m . $
 If $m \geq  n +1 $, then  
for almost all $\lambda_j , j=1, \dots, m$ the following holds:
for any $ \alpha_i  \in \mathbb{R}, i=1, \dots, n $, the following $m$ equations
\begin{equation}\label{linear system over-determined}
  \lambda_j  = \sum_{i=1}^n  A_{ij} \alpha_i  ,  j=1, \dots, m
  \end{equation}
cannot hold simultaneously. 

Sketch of Solution to Problem 3b:
There are $m $ linear equations  and $ n $ variables $\alpha_i $, 
thus intuitively no $\alpha_i, i=1, \dots, n $ can satisfy all $ m $ equations (i.e.
over-determined linear system has no solution).
For instance, an over-determined system $$ 1 = \alpha_1 ,  1 = \alpha_1 + 1 $$ does not have a solution.
However, in some special cases an over-determined linear system still has
a solution, e.g., $$ 1 = \alpha_1,  2 = 2 \alpha_1  . $$ 
Intuitively, an over-determined linear system  \eqref{linear system over-determined} can have solution in some special cases ($\lambda_j $'s are special), but for almost all $\lambda_j$'s, the system cannot have a solution. 
More specifically, \eqref{linear system over-determined} is equivalent to 
$ \lambda = A^{\top} \alpha$, which only holds when $\lambda $ lies in the
row space of $A$. Since the row space of $A$ is an $n$-dimensional subspace
of $\mathbb{R}^m$, this is a zero-measure set.
In other words, except for a zero-measure set of $\mathbb{R}^m$,
$\lambda_1, \dots, \lambda_m$ do not satisfy \eqref{linear system over-determined}. 



Till now, we have completely proved Lemma \ref{lemma::single-3} for 1-dimensional case. 



\subsection{Solving High-Dimensional Case: Preserve Zero Measure}
How to analyze the high-dimensional case? 
Changing the definition of $\alpha_i$ in \eqref{Xi definition} to the following:
\begin{equation}\label{Xi definition new}
X_i =  y_i x_i x_i^{\top}, \quad  \alpha_i = \ell'(-y_{i}f(x_{i};\bm{\theta}^{*})),
\end{equation}
Then the problem is reformulated as the following.

\textbf{Problem 4} (high-dimensional problem): 

Suppose $\mathcal{A} \subseteq \mathbb{R}$ is a finite set, and 
 $ X_i  \in \mathbb{S}^{d \times d}, i=1, \dots, n $ are given matrices (here
 $\mathbb{S}^{d \times d}$ is the set of real symmetric matrices).
If $m \geq  n +1 $, then  there exists $\lambda_j , j=1, \dots, m$ such that for all 
$A_{ij} \in \mathcal{A}, \alpha_{i} \in \mathbb{R},  \alpha_i \in \mathbb{R},
 i=1, \dots, n; j=1, \dots, m  $,
one of $M_j$ has to be non-singular, where
$$
M_j =  \lambda_j I_d  - \sum_{i=1}^n \alpha_{i} A_{ij} X_i ,  j=1, \dots, m.
$$

As discussed earlier, we only need to solve the special case that  $A_{ij}$
are fixed.  Since both $A_{ij}$ and $X_i$ are fixed, we can define a new matrix variable $B_{ij} = A_{ij} X_i$. 
The following problem becomes the major problem we need to solve in high-dimensional case.

\textbf{Problem 5} (high-dimensional problem, fixed $B_{ij}$): 
Suppose 
 $ B_{ij}  \in \mathbb{S}^{d \times d}, i=1, \dots, n; j=1, \dots, m $ are 
 arbitrary given matrices.
If $m \geq  n +1 $, then  there exists $\lambda_j , j=1, \dots, m$ such that for all 
$ \alpha_i \in \mathbb{R}, i=1, \dots, n $, one of $M_j$ has to be non-singular, where
$$
M_j =  \lambda_j I_d  - \sum_{i=1}^n \alpha_{i} B_{ij} ,  j=1, \dots, m.
$$

\subsubsection{First Attempt: Counting in Polynomial Systems of Equations}
Clearly Problem 5 is an extension of  Problem 3. 
To solve Problem 3, we transform the problem into a problem of linear systems of equations,
thus a natural idea is to transform the problem into a problem of polynomial system of equations, and hopefully some tools in algebraic geometry can
help solve the problem.

In particular, ``one of $M_j$ has to be non-singular'' 
is equivalent to ``one of $\det(M_j)$ is non-zero'', thus we can rewrite Problem 5
as follows: under the setting of Problem 5, the following $m$ polynomial equations cannot hold simultaneously:
\begin{equation}\label{determinant equal 0}
    \det( \lambda_j I_d  - \sum_{i=1}^n \alpha_{i} B_{ij} ) = 0, \quad j=1, \dots, m.
\end{equation}
These are $ m $ polynomial equations with $ n $ variables $\alpha_1, \dots, \alpha_n$;
since $m > n $, there are more equations than variables, thus
intuitively this system has no solution.

However, similar to the linear system case, over-determined polynomial system of equations
may have solution in some special case, e.g.,
$$
   1 = \alpha_1,  2 = \alpha_1^2 + \alpha_1
$$
is over-determined but it has a solution $\alpha_1 = 1$.
To solve Problem 5, we need to show that the corner cases only happen rarely, and
do not happen for almost all $\lambda_j$'s.
Unfortunately, determining the solvability of polynomial systems of equations
is a very difficult problem, and we are not aware of a universal tool of algebraic geometry
that applies to an arbitrary complicated problem.

\subsubsection{Revisit Linear Systems: Perspective of Mapping}
Our final solution requires a different perspective. 
We provide a different solution to Problem 3b from this perspective (this
trick has been used in some previous works on a different area called interference alignment).

Recall that we want to show the system of equations with variables $ \alpha_i$
$  \lambda_j  = \sum_{i=1}^n  A_{ij} \alpha_i  ,  j=1, \dots, m $
has no solution for almost all $\lambda_j$'s. 
We rewrite the equations as
\begin{equation}\label{mapping equations}
  \lambda_1 = g_1( \alpha_1, \dots, \alpha_n ), \dots, \lambda_n 
   = g_m( \alpha_1, \dots, \alpha_n) ,
\end{equation}
where $g_j$'s are mappings from $\mathbb{R}^n $ to $\mathbb{R}$. 

\textbf{Claim 1a}: 
Suppose $g_1, \dots, g_m$ are given linear mappings, and $m > n$. 
For almost all $(\lambda_1, \dots, \lambda_m)$, 
the system \eqref{mapping equations} has no solution $(\alpha_1, \dots, \alpha_n)$. 


We can define a mapping 
$$
  g( \alpha_1, \dots, \alpha_n) =  ( \sum_{i=1}^n  A_{i 1} \alpha_i , \dots, \sum_{i=1}^n  A_{i m } \alpha_i    ).
$$
This is a linear mapping that maps $\mathbb{R}^n $ to $\mathbb{R}^m$.
The image of the mapping is a linear subspace of the higher-dimensional space $\mathbb{R}^m$
that has dimension at most $n$, thus it can only occupy a zero measure set.
Thus for almost all $( \lambda_1 , \dots, \lambda_m) $, there does not exist $\alpha_i$'s such that $ g( \alpha_1, \dots, \alpha_n) = (\lambda_1, \dots, \lambda_m) $, which means that at least one of the equations in \eqref{mapping equations}  does not hold. 

\subsubsection{Math Foundation: Preserving Zero Measure}\label{subsubsec: zero measure preserving}
The core requirement of the above  proof is the following property. 

Definition (\textbf{Z-property}): Suppose a mapping $ g: \mathbb{R}^n \rightarrow \mathbb{R}^m $
where $  m > n$.  
We say $g$ satisfies Z-Property if the image of $g$ has zero measure. 

It is easy to see that the following claim holds. 

\textbf{Claim 1b}: Suppose $m > n$ and $g = (g_1, \dots, g_m)$ satisfies Z-property.
For almost all $(\lambda_1, \dots, \lambda_m)$, 
the system \eqref{mapping equations} has no solution $(\alpha_1, \dots, \alpha_n)$. 

A linear mapping $g$ satisfies Z-property (thus Claim 1a is a special case of Claim 1b).
However, not all mappings satisfy Z-property; for instance, the mapping
$$ g( a_m a_{m-1} \dots a_1. b_1 b_2 \dots ) = ( \dots a_5 a_3 a_1. b_2 b_4 \dots \; , \;
\dots a_4 a_2.b_1 b_3 \dots ) $$ maps $\mathbb{R}$ to $\mathbb{R}^2$ (this
is the classical example of mapping that shows that $\mathbb{R}$ and $\mathbb{R}^2$
have the same cardinality), and thus does not satisfy Z-property. 

Z-property is closedly related to the so-called Luzin N-property:
a function $f: [a, b] \rightarrow \mathbb{R}$ satisfies Luzin N-property 
if it maps any zero measure set to a zero measure set. 
A classical example that does NOT satisfy Luzin N-property is the Cantor function. 
It is known that if $f$ is 
Lipschitz on the compact set $[a,b]$, then $f$ satisfies Luzin N-property. 

It is not hard to prove the following result on Z-property. 

\textbf{Claim 2}: 
a mapping $ g = (g_1, \dots, g_m): \mathbb{R}^n \rightarrow \mathbb{R}^m $
satisfies Z-property if  $ g_i $ is Lipschitz continuous on any compact set of $ \mathbb{R}^n $,  for any $i$.

For instance, if all $g_i$ are polynomials, then they are Lipschitz continuous on any compact set, thus $g$ satisfies Z-property. 

Next, we need to transform Problem 5 to a form that we can utilize Z-Property. 

  






\subsubsection{Final Piece: Eigenvalue Mapping is Lipschitz}

We define a mapping 
$$
g( \alpha_1, \dots, \alpha_n) = ( \xi(  \sum_{i=1}^n \alpha_i  B_{i 1} )  , \dots,    \xi( \sum_{i=1}^n  \alpha_i B_{im}  )   ),
$$
where $\xi(Z)$ is a mapping that maps a matrix $Z$ to one of its eigenvalues. 
 This is not a mathematically rigorous definition, since $\xi$ is a one-to-many mapping, that  $\xi$ can map one $Z$ to one of its  $d$ eigenvalues; but for simplicity, we will
 use this notation. 
 
We only need to prove $g$ satisfies Z-property, i.e., it maps $\mathbb{R}^n$
to a zero measure set in $\mathbb{R}^m$. However, 
the difficulty here is that  
$g$ does not even have a closed-form expression. 
As mentioned before, an arbitrary mapping may or may not satisfy $Z$ property,
so we need to utilize a certain  property of the eigenvalue-mapping $g$.



The final piece of the proof is a (highly nontrivial) result of Lidskii which states
that the mapping from a matrix to its eigenvalues is (globally) Lipschitz continuous.
Combining this result with Claim 2 that a Lipschitz mapping
preservses zero measure, we obtain that the mapping $g$ satisfies Z-property, i.e., it
has a zero-measure image in $\mathbb{R}^m $. 
In other words, for almost all $(\lambda_1, \dots, \lambda_m)$ , the desirable property holds.
This finishes the proof of Problem 5.

\subsection{Summary of Proof Techniques}\label{subsec::Summary}

After a somewhat long journey, we derive the main proof techniques of
the proof of the key technical result Lemma \ref{lemma::single-3}. 
The problem is to show that at least one of $m$ matrices has to be
full rank. 
The main theme of the proof is to relate the problem to the solvablity of a
polynomial system of equations, and utilize the intuition that
an over-determined system has no solution except for rare cases. 

There are two major technical difficulties even if we are told that
this path can lead to a correct proof. 

The first difficulty is about the degrees of freedom (DoF).  A naive counting shows that there
are too many (DoF) for the adversary:
all weights $a_j, \bm{w}_j, j=1, \dots, m$ are variables and thus there are $O(d m )$
variables. Meanwhile, we only have $ m $ equations, so this is NOT an over-determined system
no matter how many neurons (how large $m$) we can pick.
We need some tricks to reduce the DoF for the adversary. 
Notice that the weights only affect
$ \ell'( \hat{y}_i ) $ and $ \mathbb{I}\left\{{  \bm{ w }_{j}}^{\top} x_{i}  \ge0\right\} \text{sgn}(a_j) $, and we make two observations for these two terms:
first, $ \ell'( \hat{y}_i ), i=1, \dots, n $ only have $n$ DoF since there are only $ n $ predictions;
second, $ \mathbb{I}\left\{ {  \bm{ w }_{j}}^{\top} x_{i}  \ge0\right\} \text{sgn}(a_j) $ only take finite values,
and a finite alphabet does not affect the zero measure (this is not that obvious to see, unless understanding the zero measure set argument).
Having reduced the DoF,  the 1-dimensional case becomes determining the solvability of
an over-determined linear system, which holds according to linear algebra knowledge. 

The second difficulty is to rigorously prove that the over-determined polynomial system
we face is indeed not solvable. 
Solvability of a general polynomial systems of equations is notoriously difficult,
and we have to resort to other structure of the problem.
Motivated by works in another area called intereference alignment \cite{sun2014interference},
we view the solvability of the system
from the perspective of mapping, and reduce the problem to proving a certain mapping preserves
zero measure.
Borrowing the idea of an existing result on Luzin N-property, we show that
a Lipschitz mapping preserves zero measure.
Then we utilize a classical result that the eigenvalue mapping is Lipschitz, to finish the proof.


\section{Proof of Proposition~\ref{prop::size}}

Recall that the following two assumptions are used in this paper.
\begin{assumption}[Loss function]\label{assumption::loss}
 Assume that the univariate loss function $\ell:\mathbb{R}\rightarrow\mathbb{R}_{\ge0}$ is non-decreasing and twice differentiable. Assume that there exists $\e>0$ such that any $z\in\mathbb{R}$ satisfying $\ell'(z)<\e$ always satisfies $z<0$.
\end{assumption}

\begin{assumption}[Dataset]\label{assumption::dataset}
Assume that the dataset $\mathcal{D}=\{(x_{i},y_{i})\}_{i=1}^{n}$ satisfies that for any $i,j\in[n]$, $x_{i}\neq x_{j}$ if $i\neq j$.
\end{assumption}

\begin{proposition}\label{prop::size}
Assume that $m\le n$ and $\ell$ is the logistic loss. Then there exists a dataset $\mathcal{D}$ satisfying Assumption~\ref{assumption::dataset} such that for any $\bm{\lambda}\in(0,1/2)^{m}$, the empirical loss $\lossc$ has a local minimum $\bm{\theta}$ with a training error at least $1-m/n$, i.e., $R(\bm{\theta};f)\ge 1-m/n$.
\end{proposition}

\begin{proof}
For simplicity of notation, let $z_{i}={{x_{i}}\choose{1}}$ and let $\bm{u}_{j}={{\bm{w}_{j}}\choose{b_{j}}}$. Thus, 
$$f(z;\bm{\theta})=\sum_{j=1}^{m}a_{j}\requ(\bm{u}^{\top}_{j}z)_{}$$

 We first construct the dataset as follows. Let $d$ be a sufficiently large number such that there exists $x_{1},...,x_{n}\in\mathbb{R}^{d}$ such that $\|x_{1}\|_{2}=...=\|x_{n}\|_{2}=2$, $x_{i}^{\top}x_{j}<-1$ and $x_{i}\neq x_{j}$ for any $i\neq j$. Now we consider the dataset $\mathcal{D}=\{(x_{i},y_{i})\}_{i=1}^{n}$ where $y_{1}=...y_{m}=1$ and $y_{m+1}=...=y_{n}=-1$.
We first consider a  empirical loss defined as follows
$$\tilde{L}_{m}(\bm{\theta};\bm{\lambda})=\sum_{i=1}^{m}\ell(-y_{i}f(x_{i};\bm{\theta}))+\frac{1}{3}\sum_{j=1}^{m}\lambda_{j}\left(|a_{j}|^{3}+2\|\bm{u}_{j}\|_{2}^{3}\right).$$
For any given $\bm{\lambda}\in(0,1/2)^{m}$, we construct a local minimum $\bm{\theta}^{*}$ of the empirical loss $\tilde{L}_{m}(\bm{\theta};\bm{\lambda})$ as follows. Let $\bm{u}_{j}=r_{j}z_{j}$ and let 
\begin{equation}\label{eq::11}(a^{*}_{1},...,a^{*}_{m},r^{*}_{1},...,r^{*}_{m})=\arg\min_{a_{1},...,a_{m},r_{1},...,r_{m}}\sum_{i=1}^{m}\ell(-y_{i}a_{j}(r_{j})_{+}^{2})+\frac{1}{3}\sum_{j=1}^{m}\lambda_{j}(|a_{j}|^{3}+2|r_{j}|^{3})\end{equation}
It is easy to see that if $\lambda_{j}<1/2$ for all $j\in[m]$, then $a_{1}^{*}\neq0,...,a_{m}^{*}\neq 0$ and $r_{1}^{*}> 0,...,r_{m}^{*}> 0$. Now we will show that $\bm{\theta}^{*}=(a_{1}^{*},...,a_{m}^{*},r_{1}^{*}z_{1},...,r_{m}^{*}z_{m})$ is a local minimum of the empirical loss $\tilde{L}_{m}$. Consider perturbations $\tilde{\bm{\theta}}=(\tilde{a}_{1},...,\tilde{a}_{m},\tilde{\bm{u}}_{1},...,\tilde{\bm{u}}_{m})$. Since for each $j\in[m]$, the vector $\tilde{u}_{j}$ can always be written as a linear combination of $z_{j}$ and a vector $z^{\bot}_{j}$ perpendicular to $z_{j}$. Therefore, we can rewrite $\tilde{\bm{u}}_{1},...,\tilde{\bm{u}}_{m}$ as 
 \begin{align}\tilde{\bm{u}}_{1}=\alpha_{1}z_{1}+\beta_{1}z_{1}^{\bot},...,\tilde{\bm{u}}_{m}=\alpha_{m}z_{m}+\beta_{m}z_{m}^{\bot}.
 \end{align}
Let $\delta>0$ be sufficiently small. Thus, for each $i\in[m]$ and for  $\sum_{j=1}^{m}|\alpha_{j}-r_{j}^{*}|+|\beta_{j}|<\delta$, we have 
\begin{align*}
f(x_{i};\tilde{\bm{\theta}})&=\sum_{j=1}^{m}\tilde{a}_{j}\left(\alpha_{j}z_{j}^{\top}z_{j}+\beta_{j}{z^{\bot}_{j}}^{\top}z_{j}\right)_{+}^{2}=\tilde{a}_{j}\left(\alpha_{j}\right)_{+}^{2}
\end{align*}
and for each $j\in[m]$
$$\lambda_{j}(|\tilde{a}_{j}|^{3}+\|\tilde{\bm{u}}_{j}\|_{2}^{3})=\lambda_{j}(|\tilde{a}_{j}|^{3}+2(\alpha_{j}^{2}+\beta_{j}^{2})^{3/2})\ge \lambda_{j}(|\tilde{a}_{j}|^{3}+2|\alpha_{j}|^{3}).$$
Therefore, 
\begin{align}
\tilde{L}_{m}(\tilde{\bm{\theta}};{\bm{\lambda}})&=\sum_{i=1}^{m}\ell(-y_{i}f(x_{i};\bm{\theta}))+\frac{1}{3}\sum_{j=1}^{m}\lambda_{j}\left(|\tilde{a}_{j}|^{3}+2\|\tilde{\bm{u}}_{j}\|_{2}^{3}\right)\\
&\ge\sum_{i=1}^{m}\ell(-y_{i}\tilde{a}_{j}(\alpha_{j})_{+}^{2})+\frac{1}{3}\sum_{j=1}^{m}\lambda_{j}(|\tilde{a}_{j}|^{3}+2|\alpha_{j}|^{3})
\end{align}
Since $(\tilde{a}_{1},...,\tilde{a}_{m},\alpha_{1},...,\alpha_{m})$ is a perturbation  of $\bm{\theta}^*=(a_{1}^{*},...,a_{m}^{*},r_{1}^{*},...r_{m}^{*})$ and $\bm{\theta}^*$ is a local minimum of the loss defined in Eq.~\eqref{eq::11}, then
\begin{align*}
&\sum_{i=1}^{m}\ell(-y_{i}\tilde{a}_{j}(\alpha_{j})_{+}^{2})+\frac{1}{3}\sum_{j=1}^{m}\lambda_{j}(|\tilde{a}_{j}|^{3}+2|\alpha_{j}|^{3})\\
&\ge \sum_{i=1}^{m}\ell(-y_{i}a^{*}_{j}(r^{*}_{j})_{+}^{2})+\frac{1}{3}\sum_{j=1}^{m}\lambda_{j}(|a^{*}_{j}|^{3}+2|r^{*}_{j}|^{3})
=\tilde{L}_{m}(\bm{\theta}^{*};\bm{\lambda})
\end{align*}
Therefore, for any sufficiently small perturbation on $\bm{\theta}^{*}$, we should always have 
$$\tilde{L}_{m}(\tilde{\bm{\theta}};\bm{\lambda})\ge \tilde{L}_{m}(\bm{\theta}^{*};\bm{\lambda}).$$
Therefore, $\bm{\theta}^{*}$ is a local minimum of the empirical loss $\tilde{L}_{m}({\bm{\theta}};\bm{\lambda})$.

Now we consider the following empirical loss 
$$L_{n}(\bm{\theta};\bm{\lambda})=\sum_{i=1}^{n}\ell(-y_{i}f(z_{i};\bm{\theta}))+\frac{1}{3}\sum_{j=1}^{m}\lambda_{j}(|a_{j}|^{3}+2\|\bm{u}_{j}\|_{2}^{3}).$$
Next, we will prove that the set of parameters $\bm{\theta}^{*}$ defined above is a local minimum of the empirical loss $\loss$. For $i=m+1,...,n$, it is easy to see that for any sufficiently small perturbation $\Delta\bm{\theta}$ on the vector $\bm{\theta}^{*}$, 
$$f(z_{i};\tilde{\bm{\theta}})=\sum_{j=1}^{m}\tilde{a}_{j}\requ(\Delta{\bm{u}}^{\top}_{j}z_{i}+{\bm{u}_{j}^{*}}^{\top}z_{i})=\sum_{j=1}^{m}\tilde{a}_{j}\requ(\Delta{\bm{u}}^{\top}_{j}z_{i}+r_{j}^{*}z^{\top}_{j}z_{i})=0,$$
 since $r_{j}^{*}>0$ and $z_{i}^{\top}z_{j}<0$ for any $i\neq j$.
Therefore, for any  perturbed parameter vector $\tilde{\bm{\theta}}$ of sufficiently small magnitude, we always have 
\begin{align*}
L_{n}(\tilde{\bm{\theta}};\bm{\lambda})=(n-m)\ell(0)+\tilde{L}_{m}(\tilde{\bm{\theta}};\bm{\lambda})\ge (n-m)\ell(0)+\tilde{L}_{m}({\bm{\theta}}^{*};\bm{\lambda})=L_{n}(\bm{\theta}^{*};\bm{\lambda}).
\end{align*}
Therefore, $\bm{\theta}^{*}$ is a local minimum of the empirical loss $\lossc$. Furthermore, since for  $i=m+1,...,n$, we have 
$$f(z_{i};\bm{\theta}^{*})=\sum_{j=1}^{m}a_{j}^{*}\requ({\bm{u}_{j}^{*}}^{\top}z_{i})=\sum_{j=1}^{m}a_{j}^{*}\requ(r_{j}^{*}z_{j}^{\top}z_{i})=0,$$
then the network $f(x;\bm{\theta}^{*})$ misclassifies at least $n-m$ samples. Thus, $$R_{n}(\bm{\theta}^{*};f)\ge 1-\frac{m}{n}.$$
\end{proof}


\section{Proof of Proposition~\ref{prop::activation}}
\begin{proposition}\label{prop::activation}
Assume $m\ge 2d+4$, $\bm{\lambda}\in\mathbb{R}_{+}^m$  and $\lambda_{i}\neq \lambda_{j}$ for any $i\neq j$.  Under Assumption~\ref{assumption::loss}, if the dataset is quadratically separable, then both of the following statements are true:
\begin{itemize}[leftmargin=*]
\item[(1)] the empirical loss $L_{n}(\bm{\theta};\bm{\lambda})$ is coercive.
\item[(2)] every local minimum $\bm{\theta}^{*}$ of the loss $\lossc$ achieves zero training error, i.e., $R_{n}(\bm{\theta^{*}};f)=0$.
\end{itemize}
\end{proposition}

\begin{proof}
(1) By definition, it is easy to check that the loss function $\lossc$ is coercive.

(2) We first prove that if $\lambda_{i}\neq \lambda_{j}$, then at every local minimum, the quadratic network always has an inactive neuron. Assume $\bm{\theta}^{*}$ is a critical point of $L(\bm{\theta};\bm{\lambda})$. By the definition of the critical point, we have 
\begin{align}
\frac{\partial \loss}{\partial a_{j}}&=\sum_{i=1}^{n}\ell'(-y_{i}f(x_{i};\bm{\theta}^{*}))(-y_{i})({\bm{w}_{j}^{*}}^{\top}x_{i}+b^{*}_{j})_{}^{2}+\lambda_{j}|a_{j}|a_{j}=0\label{eq::prop-2-1}\\
\nabla_{\bm{w}_{j}}\loss&=2\sum_{i=1}^{n}\ell'(-y_{i}f(x_{i};\bm{\theta}^{*}))(-y_{i})a_{j}^{*}({\bm{w}_{j}^{*}}^{\top}x_{i}+b^{*}_{j})_{}{{x_{i}}\choose{1}}\label{eq::prop-2-2}\\
&\quad+2\lambda_{j}\sqrt{\|\bm{w}^{*}_{j}\|_{2}^{2}+{{b}_{j}^{*}}^{2}}{{\bm{w^{*}_{j}}}\choose{b^{*}_{j}}}=\bm{0}_{d+1}\notag
\end{align}

Multiplying the both sides of Eq.~\eqref{eq::prop-2-1} by $a^{*}$ and taking the inner product of the both sides of Eq.~\eqref{eq::prop-2-2} with the vector $(\bm{w}_{j}^{*},b_{j}^{*})$, we obtain \begin{align}|a_{j}^{*}|=(\|\bm{w}^{*}_{j}\|_{2}^{2}+{b_{j}^{*}}^{2})^{1/2},\quad \forall j\in[m].\end{align} Therefore, we can only have the following two cases: (1) the network $f(x;\bm{\theta}^{*})$ has an inactive neuron; (2) all neurons of the network $f(x;\bm{\theta}^{*})$ are active. For case (1), we proved the lemma. For case (2), since $a_{j}^{*}\neq0$ for all $j\in[m]$, then dividing both sides of Eq.~\eqref{eq::prop-2-2} by $|a_{j}^{*}|$, we obtain that for $\forall j\in[m],$
\begin{align}-\sgn(a_{j}^{*})\sum_{i=1}^{n}\ell'(-y_{i}f(x_{i};\bm{\theta}^{*}))y_{i}({\bm{w}_{j}^{*}}^{\top}x_{i}+b^{*}_{j})_{}{{x_{i}}\choose{1}}+\lambda_{j}{{\bm{w^{*}_{j}}}\choose{b^{*}_{j}}}=\bm{0}. \end{align}

We can rewrite it as 
\begin{align}\label{eq::prop-2-3}
M_{j}{{\bm{w^{*}_{j}}}\choose{b^{*}_{j}}}=\bm{0},\quad \forall j\in[m],
\end{align}
where square matrices $M_{1},...,M_{m}$ are defined as 
\begin{align*}M_{j}=-\sgn(a_{j}^{*})\sum_{i=1}^{n}\ell'(-y_{i}f(x_{i};\bm{\theta}^{*}))y_{i}{{x_{i}}\choose{1}}\left(\begin{matrix}x_{i}^{\top}& 1\end{matrix}\right)+\lambda_{j}\bm{I}_{d+1},\quad \forall j\in[m].\end{align*}

Since $m\ge 2d+4$, then there exists at least $d+2$ neurons such that their coefficients have the same sign. Without loss of generality, we assume that $\sgn(a_{1}^{*})=...=\sgn(a^{*}_{d+2})\triangleq a_{0}$.  Furthermore, we notice that the matrix 
$$M=a_{0}\sum_{i=1}^{n}\ell'(-y_{i}f(x_{i};\bm{\theta}^{*}))y_{i}{{x_{i}}\choose{1}}\left(\begin{matrix}x_{i}^{\top}& 1\end{matrix}\right)$$
is independent of the index $j$ and  can have at most $d+1$ different eigenvalues. Therefore, since $\lambda_{i}\neq\lambda_{j}$ for any $i\neq j$ and $M_{j}=-M+\lambda_{j}\bm{I}_{d+1}$, then one of the matrices $M_{1},...,M_{d+2}$ is non-singular. From Eq.~\eqref{eq::prop-2-3}, it follows that one of  the vectors $(\bm{w}_{j}^{*},b_{j}^{*})$ is a zero vector and thus the network $f(x;\bm{\theta}^{*})$ has an inactive neuron. Thus, at every local minimum of the loss $\lossc$, the quadratic network always has an inactive neuron. 

Without loss of generality, we assume that  $a_{1}^{*}=0, \|\bm{w}_{1}^{*}\|_{2}=0$ and $b_{1}^{*}=0$. Since $\bm{\theta}^{*}$ is a local minimum of the empirical loss $\lossc$, then there exists a $\delta>0$, such that for any $\tilde{\bm{\theta}}:\|\tilde{\bm{\theta}}-\bm{\theta}^{*}\|<2\delta$, $L(\tilde{\bm{\theta}};\bm{\lambda})\ge L(\bm{\theta}^{*};\bm{\lambda})$. Now we choose the perturbation where we only perturb parameters $a_{1},\bm{w}_{1},b_{1}$. Let $\tilde{a}_{1}=\delta\sgn(\tilde{a}_{1})$, $(\tilde{\bm{w}}_{1},\tilde{b})=(\delta \bm{u},\delta v)$ for arbitrary $\bm{u},v:\|\bm{u}\|_{2}^{2}+v^{2}=1$ and $(\tilde{a_{j}},\tilde{\bm{w}}_{j},\tilde{b}_{j})=(a_{j}^{*},\bm{w}^{*}_{j},b_{j}^{*})$ for $j\neq 1$. By the second order Taylor expansion, we obtain that, for any $\sgn(\tilde{a}_{j})$ and any $(\bm{u},v):\|\bm{u}\|_{2}^{2}+v^{2}=1$,
\begin{align*}
&L(\tilde{\bm{\theta}};\bm{\lambda})=\sum_{i=1}^{n}\ell\left(-y_{i}f(x_{i};\tilde{\bm{\theta}})\right)+\frac{1}{3}\sum_{j=1}^{m}\lambda_{j}\left[|\tilde{a}_{j}|^{3}+2(\|\tilde{\bm{w}}_{j}^{}\|_{2}^{2}+{\tilde{b}_{j}}^{2})^{3/2}\right]\\
&=\loss-\sum_{i=1}^{n}\ell'(-y_{i}f(x_{i};\bm{\theta}^{*}))y_{i}\delta^{3}\sgn(\tilde{a}_{1})(\bm{u}^{\top}x_{i}+v)_{}^{2}+R(\delta, \mathcal{D},\bm{u},v)\delta^{6}+\lambda_{1}\delta^{3}\\
&\ge\loss,
\end{align*}
where $R(\delta,\mathcal{D},\bm{u},v)$ is the second order remaining term in the Taylor expansion.
This further indicates that for any $(\bm{u},v):\|\bm{u}\|_{2}^{2}+v^{2}=1$, we have 
$$\left|\sum_{i=1}^{n}\ell'(-y_{i}f(x_{i};\bm{\theta}^{*}))(-y_{i})(\bm{u}^{\top}x_{i}+v)_{}^{2}\right|<\lambda_{1}<\lambda_{0}.$$
From the assumption that the dataset is quadratically separable,it follows that there exists a single-layered quadratic network $p(x;\bm{\rho})=\sum_{i=1}^{K}\alpha_{j}(\bm{\omega}_{j}^{\top}x_{i}+\beta_{j})_{}^{2}$ of size $K$ such that $\min_{i}y_{i}p(x_{i};\bm{\rho})>0$, where the vector $\rho$ contains all parameters in the network $p$. 
 Let $$\tilde{\lambda}\triangleq\max_{\bm{\rho}:\|\bm{\rho}\|_{2}=1}\min_{i\in[n]}y_{i}p(x_{i};\bm{\rho})>0 \quad\text{and}\quad \bm{\rho}^{*}=\arg\max_{\bm{\rho}:\|\bm{\rho}\|_{2}=1}\min_{i\in[n]}y_{i}p(x_{i};\bm{\rho}).$$
Let $\lambda_{0}=\varepsilon\tilde{\lambda}$, then 
\begin{align*}
\tilde{\lambda}\sum_{i=1}^{n}\ell'(-y_{i}f(x_{i};\bm{\theta}^{*}))<\left|\sum_{i=1}^{n}\ell'(-y_{i}f(x_{i};\bm{\theta}^{*}))(-y_{i})p(x_{i};\bm{\rho}^{*})\right|<\lambda_{0}=\e\tilde{\lambda}.
\end{align*}
Since $\ell$ is non-decreasing, then for each $i\in[n]$, $\ell'(-y_{i}f(x_{i};\bm{\theta}^{*}))<\e$. By Assumption~\ref{assumption::loss}, this further indicates that  $y_{i}f(x_{i};\bm{\theta}^{*})>0$ for $\forall i\in[n]$. Therefore, the local minimum achieves $\bm{\theta}^{*}$ zero training error, i.e., $R_{n}(\bm{\theta}^{*})=0$.
\end{proof}


\section{Lemma~\ref{lemma::multi-2}}

Starting from this appendix, the rest of the appendix is devoted
to the technical results needed for proving Theorem \ref{thm::multi},
and the proof of Theorem \ref{thm::multi} provided in Appendix \ref{appendix::multi}. 
In each appendix, we will mainly discuss one lemma or claim. 
The big picture shall be found in Appendix \ref{appendix::multi}. 

\begin{lemma}\label{lemma::multi-2}
Let $\mathcal{A}$ denote a set of finite dimension. Given a vector $\bm{\lambda}\in\mathbb{R}^{m}$, a matrix $\bm{Z}=(\bm{z}_{1},..., \bm{z}_{n})\in\mathbb{R}^{d\times n}$ and a matrix $\bm{A}=(A_{ij})\in\mathcal{A}^{n\times m}$, define matrices $\bm{M}_{1},...,\bm{M}_{n}$ as 
$$\bm{M}_{j}=-\sum_{i=1}^{n}\bm{z}_{i}\bm{z}_{i}^{\top}A_{ij}+\lambda_{j} \bm{I}_{d}, \quad j\in[m]$$
Assume $m>dn $. Given a matrix $\bm{A}$, then there exists a zero measure set $\mathcal{C}(\bm{A})\subset\mathbb{R}^{m}$ depending on $\bm{A}$ such that for any $\bm{\lambda}\notin \mathcal{C}$ and $\bm{Z}\in\mathbb{R}^{d\times n}$,  one of the matrices $\bm{M}_{1},...,\bm{M}_{n}$ is non-singular. 
\end{lemma}

\begin{proof}
Define a map $\bm{g}:\mathbb{R}^{d\times n}\rightarrow \mathbb{R}^{d}$ as 
$$\bm{g}(\bm{Z};\bm{A})=(g_{1}(\bm{Z};\bm{A}),...,g_{d}(\bm{Z};\bm{A}))=\Lambda\left(\sum_{i=1}^{n}\bm{z}_{i}\bm{z}_{i}^{\top}A_{ij}\right).$$
Given a matrix $\bm{A}$, let the set $\mathcal{C}(\bm{A})$ denote any possible $\bm{\lambda}\in\mathbb{R}^{m}$ such that all matrices $\bm{M}_{1},...,\bm{M}_{n}$ are singular. This means that for any $\lambda\in\mathcal{C}$, there exists  a matrix $\bm{Z}=(\bm{Z}_{1},..., \bm{Z}_{n})\in\mathbb{R}^{d\times n}$ and a series of indices $i_{1}\in[n],...,i_{m}\in[n]$ such that for each $k\in[m]$, $\lambda_{k}$ is the $i_{k}$-th smallest eigenvalue of the matrix $\sum_{i=1}^{n}\bm{z}_{i}\bm{z}_{i}^{\top}A_{ik}$.
Thus, we can write the set $\mathcal{C}$ as follows 
$$\mathcal{C}(\bm{A})=\bigcup_{i_{1}\in[n]}\bigcup_{i_{2}\in[n]}...\bigcup_{i_{m}\in[n]}\left\{(g_{i_{1}}(\bm{Z};\bm{A}),...,g_{i_{m}}(\bm{Z};\bm{A}))|\bm{Z}\in\mathbb{R}^{d\times n}\right\}.$$ 

Let $R_{k}=\{\bm{Z}\in\mathbb{R}^{d\times n}:\|\bm{Z}\|_{F}\le k+1\}$ for $k\in\mathbb{N}$. Thus, it is easy to see that $\mathbb{R}^{d\times n}=\bigcup_{i\in\mathbb{N}}R_{i}$.
Given the matrix $\bm{A}$,  from Lidskii's theorem, it follows that for each $i_{1}\in[n],...,i_{m}\in[n]$ and for each $k\in\mathbb{N}$, the map $(g_{i_{1}},...,g_{i_{m}}):R_{k}\rightarrow\mathbb{R}^{m}$ is Lipschitz. Furthermore, since $m\ge n+1$, then for each $i_{1}\in[n],...,i_{m}\in[n]$ and for each $k\in\mathbb{N}$, the set $\left\{(g_{i_{1}}(\bm{Z}),...,g_{i_{m}}(\bm{Z}))|\bm{Z}\in R_{k}\right\}$ has zero measure in $\mathbb{R}^{m}$. Since 
$$\left\{(g_{i_{1}}(\bm{Z};\bm{A}),...,g_{i_{m}}(\bm{Z};\bm{A}))|\bm{Z}\in\mathbb{R}^{d\times n}\right\}=\bigcup_{k\in \mathbb{N}}\left\{(g_{i_{1}}(\bm{Z};\bm{A}),...,g_{i_{m}}(\bm{Z};\bm{A}))|\bm{Z}\in R_{k}\right\},$$
then the set $\left\{(g_{i_{1}}(\bm{Z};\bm{A}),...,g_{i_{m}}(\bm{Z};\bm{A}))|\bm{Z}\in\mathbb{R}^{d\times n}\right\}$ has zero measure in $\mathbb{R}^{m}$. Furthermore, since the set $\mathcal{C}(\bm{A})$ is a union of a finite number of zero measure set in $\mathbb{R}^{m}$, then the set $\mathcal{C}$ has zero measure in $\mathbb{R}^{m}$.
\end{proof}

\begin{corollary}
Under conditions in Lemma~\ref{lemma::multi-2}, there exists a zero measure set $\mathcal{C}\subset\mathbb{R}$ such that for any $\bm{\lambda}\notin \mathcal{C}$, any $\bm{Z}\in\mathbb{R}^{d\times n}$ and any $\mathcal{A}\in\mathcal{A}^{n\times m}$, one of the matrices $\bm{M}_{1},...,\bm{M}_{n}$ is non-singular.  
\end{corollary}
\begin{proof}
From Lemma~\ref{lemma::multi-2}, it follows that, given the matrix $\bm{A}\in\mathcal{A}^{n\times m}$, there exists a zero measure set $\mathcal{C}(\bm{A})$ such that for any $\bm{\lambda}\notin\mathcal{C}(\bm{A})$ and any $\bm{Z}\in\mathbb{R}^{d\times n}$, one of the matrices $\bm{M}_{1},...,\bm{M}_{n}$ is non-singular. Now we define the set $$\mathcal{C}=\bigcup_{\bm{A}\in\mathcal{A}^{n\times m}}\mathcal{C}(\bm{A}).$$
Since the set $\mathcal{A}$ is a finite dimensional set, then the set $\mathcal{C}$ is a zero measure set. Therefore, there exists a zero measure set $\mathcal{C}$ such that for any $\bm{\lambda}\notin\mathcal{C}$, any $\bm{Z}\in\mathbb{R}^{d\times n}$ and any $\bm{A}\in \mathcal{A}^{n \times m}$, one of the matrices $\bm{M}_{1},...,\bm{M}_{n}$ is non-singular. 
\end{proof}


\section{Lemma~\ref{lemma::multi-1}}

\begin{lemma}\label{lemma::multi-1}
Assume $m> (d+ls)n$.  There exists a zero measure set $\mathcal{C}\subset\mathbb{R}^{m}$ such that for any $\bm{\lambda}\notin \mathcal{C}$, at every critical point $\bm{\theta}^{*}$ of the empirical loss $L(\bm{\theta};\bm{\lambda})$, the neural network $f(x;\bm{\theta}^{*})$ always has an inactive neuron in the $l$-th layer, i.e., $\exists j\in[m]$ s.t. $(a^{*}_{j}, \|\bm{w}_{j}^{*}\|_{2},b_{j}^{*})=(0,0,0)$.
\end{lemma}
\begin{proof}
For simplicity of notation, let the $\bm{h}(x;\bm{\theta})=\bm{h}^{(l-1)}(x;\bm{\theta})$ denote the output of the $(l-1)$-th layer under the parameter $\bm{\theta}$. Then the output of the neural network can be rewritten as 
$$f(x)=\sum_{j=1}^{m}a_{j}(\bm{w}_{j}^{\top}\bm{h}(x_{i})+b_{j})_{+}^{2}.$$
Assume $\bm{\theta}^{*}$ is a critical point of $L(\bm{\theta};\bm{\lambda})$. By the definition of the critical point, we have 
\begin{align}
\frac{\partial \loss}{\partial a_{j}}&=\sum_{i=1}^{n}\ell'(-y_{i}f(x_{i};\bm{\theta}^{*}))(-y_{i})({\bm{w}_{j}^{*}}^{\top}\bm{h}(x_{i};\bm{\theta}^{*})+b^{*}_{j})_{+}^{2}+\lambda_{j}|a_{j}^{*}|a^{*}_{j}=0\label{eq::multi-1-1}\\
\nabla_{\bm{w}_{j}}\loss&=2\sum_{i=1}^{n}\ell'(-y_{i}f(x_{i};\bm{\theta}^{*}))(-y_{i})a_{j}^{*}({\bm{w}_{j}^{*}}^{\top}\bm{h}(x_{i};\bm{\theta}^{*})+b^{*}_{j})_{+}{{\bm{h}(x_{i};\bm{\theta}^{*})}\choose{1}}\notag\\
&\quad+2\lambda_{j}\sqrt{\|\bm{w}^{*}_{j}\|_{2}^{2}+{{b}_{j}^{*}}^{2}}{{\bm{w^{*}_{j}}}\choose{b^{*}_{j}}}=\bm{0}_{d+1}\label{eq::multi-1-2}
\end{align}
By multiplying $a_{j}^{*}$ on the both sides of Eq.~\eqref{eq::multi-1-1} and taking the inner product of the both sides of Eq.~\eqref{eq::multi-1-2} with the vector $(\bm{w}_{j}^{*},b_{j}^{*})$, we have 
\begin{equation}|a_{j}^{*}|=(\|\bm{w}^{*}_{j}\|_{2}^{2}+{b_{j}^{*}}^{2})^{1/2},\quad \forall j\in[m].\end{equation}
Thus, we only have two cases: (1) there is an inactive neuron in the last layer of the network $f$; (2) all neurons in the last layer of the network $f$ are active.  
For case (1), we automatically finished the proof. For case (2), since all neurons are active, i.e., $|a_{j}^{*}|=(\|\bm{w}^{*}_{j}\|_{2}^{2}+{b_{j}^{*}}^{2})^{1/2}>0$ for all $j\in[m]$, then after dividing both sides of Eq.~\eqref{eq::multi-1-2} by $|a_{j}^{*}|$, we obtain that for $ \forall j\in[m]$
\begin{equation}\label{eq::multi-1-3}
-\sgn(a_{j}^{*})\sum_{i=1}^{n}\ell'(-y_{i}f(x_{i};\bm{\theta}^{*}))y_{i}({\bm{w}_{j}^{*}}^{\top}\bm{h}(x_{i};\bm{\theta}^{*})+b^{*}_{j})_{+}{{\bm{h}(x_{i};\bm{\theta}^{*})}\choose{1}}+\lambda_{j}{{\bm{w^{*}_{j}}}\choose{b^{*}_{j}}}=\bm{0}_{d+1}.
\end{equation}
We can rewrite it as 
\begin{equation}\label{eq::multi-1-4}
M_{j}{{\bm{w^{*}_{j}}}\choose{b^{*}_{j}}}=\bm{0}, \forall j\in[m],
\end{equation}
where matrices $M_{1},...,M_{n}$ are defined as 
\begin{align}
M_{j}&=-\sgn(a_{j}^{*})\sum_{i=1}^{n}\ell_{i}'y_{i}\mathbb{I}\left\{{\bm{w}_{j}^{*}}^{\top}\bm{h}(x_{i};\bm{\theta}^{*})+b^{*}_{j}\ge0\right\}{{\bm{h}(x_{i};\bm{\theta}^{*})}\choose{1}}\left(\begin{matrix}\bm{h}(x_{i};\bm{\theta}^{*})^{\top}& 1\end{matrix}\right)\\
&\quad+\lambda_{j}\bm{I}_{d+1}.\notag
\end{align}
for $j\in[n]$  where $\ell_{i}' = \ell'(-y_{i}f(x_{i};\bm{\theta}^{*}))$.
Then from Lemma~\ref{lemma::multi-2}, it follows that there exists a zero-measure set $\mathcal{C}\subset\mathbb{R}^{m}$ such that for any $\lambda\notin \mathcal{C}$, one of the matrices $M_{j}$ is non-singular. Therefore, by Eq.~\eqref{eq::multi-1-4}, we obtain that one of the vector $(\bm{w}_{j}^{*},b_{j}^{*})$ is a zero vector and thus the network $f(x;\bm{\theta}^{*})$ has an inactive neuron in the $l$-th layer (last layer).
\end{proof}



\section{Lemma~\ref{lemma::multi-3}}
Recall that $\bm{h}^{(i)}(x;\bm{\theta})$ denotes the output of the $i$-th layer and that the first $l-1$ layers are convolutional layers.

\begin{lemma}\label{lemma::multi-3}
Assume that $s\ge 1$ and $\lambda_{c}>0$. Under Assumption~\ref{assumption::dataset}, then for any $\bm{\lambda}\in\mathbb{R}_{+}^{m}$ and any two different samples $({x}_{i},y_{i})$ and $({x}_{j},y_{j})$ in the dataset $\mathcal{D}$, we should have $\bm{h}^{(l-1)}(x_{i};\bm{\theta}^{*})\neq \bm{h}^{(l-1)}(x_{j};\bm{\theta}^{*})$ for any $\bm{v}_{1}\neq \bm{0},...,\bm{v}_{l-1}\neq \bm{0}$.
\end{lemma}

\begin{proof}
We only need to prove that if there exists some vectors $\bm{v}_{1}\neq \bm{0},...,\bm{v}_{l-1}\neq \bm{0}$ such that $\bm{h}^{(l-1)}(x_{1};\bm{\theta})= \bm{h}^{(l-1)}(x_{2};\bm{\theta})$, then ${x_{1}}={x_{2}}$. Since for each $k\in[l-1]$,
$$\bm{h}^{(k)}(x;\bm{\theta})=\bm{\sigma}\left(\bm{v}_{k}\ast \bm{h}^{(k-1)}(x;\bm{\theta})\right),$$
then we only need to prove that, for each $k\in[l-1]$, if $\bm{h}^{(k)}(x_{1};\bm{\theta})= \bm{h}^{(k)}(x_{2};\bm{\theta})$, then $\bm{h}^{(k-1)}(x_{1};\bm{\theta})= \bm{h}^{(k-1)}(x_{2};\bm{\theta})$. 

Since the activation $\bm{\sigma}$ of Leaky-ReLU  is strictly increasing, then if $$\bm{h}^{(k)}(x_{1};\bm{\theta})= \bm{h}^{(k)}(x_{2};\bm{\theta}),$$ we should have 
$$\bm{v}_{k}\ast \bm{h}^{(k-1)}(x_{1};\bm{\theta})= \bm{v}_{k}\ast \bm{h}^{(k-1)}(x_{2};\bm{\theta}).$$
It is easy to see that, given a weight vector $\bm{v}$, the map $\bm{z}\mapsto\bm{v}\ast \bm{z}$ is a linear map. Furthermore, from Claim~\ref{lemma::multi-4} (proved below), it follows that if $\bm{v}\neq \bm{0}$ and $\bm{v}\ast \bm{z}=\bm{0}$, then $\bm{z}=\bm{0}$. Therefore, by assumption that $\bm{v}_{k}\neq \bm{0}$ for $\forall k\in[l-1]$, we should have $$\bm{h}^{(k-1)}(x_{1};\bm{\theta})-\bm{h}^{(k-1)}(x_{2};\bm{\theta})=\bm{0}.$$
By induction, we have $\bm{h}^{(0)}(x_{1};\bm{\theta})-\bm{h}^{(0)}(x_{2};\bm{\theta})={x}_{1}-{x}_{2}=\bm{0}$ and thus complete the proof.
\end{proof}


\begin{claim}\label{lemma::multi-4}
For two vectors $\bm{v}$ and $\bm{z}$, if $\bm{v}\ast \bm{z}=\bm{0}$ and $\bm{v}\neq \bm{0}$, then $\bm{z}=\bm{0}$.
\end{claim}

\begin{proof}
Let $\bm{v}=(\bm{v}(1),...,\bm{v}(d_{v}))$ and $\bm{z}=(\bm{z}(1),...,\bm{z}(d_{z}))$, then 
\begin{align*}
&\bm{v}\ast \bm{z}\\
&=\left(
\begin{matrix}
\bm{v}(d_{v})	&0			&0			&0			&...			&0			&0			&0	\\
\bm{v}(d_{v-1})	&\bm{v}(d_{v})	&0			&0			&...			&0			&0			&0	\\
\bm{v}(d_{v-2})	&\bm{v}(d_{v-1})&\bm{v}(d_{v})	&0			&...			&0			&0			&0\\
\bm{v}(d_{v-3})	&\bm{v}(d_{v-2})&\bm{v}(d_{v-1})&\bm{v}(d_{v})&...			&0			&	0		&0\\
...	&...		&...			&...			&...			&...			&...		&...	\\
0			&0			&0			&0			&...			&\bm{v}(d_{1})	&\bm{v}(d_{2})&\bm{v}(d_{3})	\\
0			&0			&0			&0			&...			&0			&\bm{v}(d_{1})	&\bm{v}(d_{2})\\
0			&0			&0			&0			&...			&0			&	0		&\bm{v}(d_{1})
\end{matrix}
\right)_{(d_{v}+d_{z}-1)\times d_{z}}\bm{z}\\
&\triangleq \bm{V}\bm{z}
\end{align*}
It is straightforward to see that if the vector $\bm{v}\neq 0$, then the matrix $\bm{V}\in\mathbb{R}^{(d_{v}+d_{z}-1)\times d_{z}}$ is always a full rank matrix of rank $d_{z}$. Therefore, if $\bm{v}\ast \bm{z}=\bm{0}$ and $\bm{v}\neq \bm{0}$, then we have $\bm{z}=\bm{0}$.
\end{proof}


\section{Lemma~\ref{lemma::multi-5}}

\begin{lemma}\label{lemma::multi-5}
Assume that Assumption~\ref{assumption::loss} holds. Let $\bm{\theta}^{*}$ denote a local minimum of the loss $L_{n}(\bm{\theta};\bm{\lambda})$ and assume that there exists an inactive neuron in the last layer and $\|\bm{v}^{*}_{k}\|_{2}\ge1$ holds for all $k\in[l-1]$. Then there exists $\lambda_{0}=\lambda_{0}(\mathcal{D},\ell)>0$ such that for all $\bm{\lambda}\in(0,\lambda_{0})^{m}$ and for any $\lambda_{c}>0$, we have $y_{i}f(x_{i};\bm{\theta}^{*})>0, \forall i\in[n]. $
\end{lemma}
\begin{proof}
Since $\bm{\theta}^{*}$ is a local minimum of the empirical loss $\lossc$ and the neural network $f(x;\bm{\theta}^{*})$ has an inactive neuron in the last layer. Without loss of generality, we assume that  $a_{1}^{*}=0, \|\bm{w}_{1}^{*}\|_{2}=0$ and $b_{1}^{*}=0$. Since $\bm{\theta}^{*}$ is a local minimum of the empirical loss $\lossc$, then there exists a $\delta>0$, such that for any $\tilde{\bm{\theta}}:\|\tilde{\bm{\theta}}-\bm{\theta}^{*}\|<2\delta$, $L_{n}(\tilde{\bm{\theta}};\bm{\lambda})\ge L_{n}(\bm{\theta}^{*};\bm{\lambda})$. Now we choose the perturbation where we only perturb $a_{1},\bm{w}_{1},b_{1}$. Let $\tilde{a}_{1}=\delta\sgn(\tilde{a}_{1})$, $(\tilde{\bm{w}}_{1},\tilde{b})=(\delta \bm{u},\delta v)$ where $\|\bm{u}\|_{2}^{2}+v^{2}=1$, $(\tilde{a_{j}},\tilde{\bm{w}}_{j},\tilde{b}_{j})=(a_{j}^{*},\bm{w}^{*}_{j},b_{j}^{*})$ for $j\neq 1$ and $\tilde{\bm{v}}_{k}=\bm{v}^{*}_{k}$.
Recall that the vector $\bm{h}^{(l-1)}(x;\bm{\theta})$ denotes the output of the $(l-1)$-th layer. 
Thus, by the second order Taylor expansion, we have that for any $\sgn(\tilde{a}_{j})$ and any $(\bm{u},v):\|\bm{u}\|_{2}^{2}+v^{2}=1$,
\begin{align*}
L_{n}(\tilde{\bm{\theta}};\bm{\lambda})&=\sum_{i=1}^{n}\ell\left(-y_{i}f(x_{i};\bm{\theta}^{*})-y_{i}\delta^{3}\sgn(\tilde{a}_{1})(\bm{u}^{\top}\bm{h}^{(l-1)}(x_{i};\bm{\theta}^{*})+v)_{+}^{2}\right)\\
&\quad+\frac{1}{3}\sum_{j=1}^{m}\lambda_{j}\left[|a_{j}^{*}|^{3}+2(\|\bm{w}_{j}^{*}\|_{2}^{2}+{b_{j}^{*}}^{2})^{3/2}\right]+\frac{\lambda_{1}}{3}\left[|\tilde{a}_{j}|^{3}+2(\|\tilde{\bm{w}}_{j}\|_{2}^{2}+{\tilde{b}_{j}}^{2})^{3/2}\right]\\
&=\sum_{i=1}^{n}\ell\left(-y_{i}f(x_{i};\bm{\theta}^{*})\right)-\sum_{i=1}^{n}\ell'_{i}y_{i}\delta^{3}\sgn(\tilde{a}_{1})(\bm{u}^{\top}\bm{h}^{(l-1)}(x_{i};\bm{\theta}^{*})+v)_{+}^{2}\\
&\quad+R(\delta, \mathcal{D},\bm{u},v)\delta^{6}+\frac{1}{3}\sum_{j=1}^{m}\lambda_{j}\left[|a_{j}^{*}|^{3}+2(\|\bm{w}_{j}^{*}\|_{2}^{2}+{b_{j}^{*}}^{2})^{3/2}\right]+\lambda_{1}\delta^{3}\\
&\ge\sum_{i=1}^{n}\ell\left(-y_{i}f(x_{i};\bm{\theta}^{*})\right)+\frac{1}{3}\sum_{j=1}^{m}\lambda_{j}\left[|a_{j}^{*}|^{3}+2(\|\bm{w}_{j}^{*}\|_{2}^{2}+{b_{j}^{*}}^{2})^{3/2}\right]=\loss
\end{align*}
where $\ell'_{i} = \ell'(-y_{i}f(x_{i};\bm{\theta}^{*}))$ and $R(\delta, \mathcal{D},\bm{u},v)$ is the second order remaining term in the Taylor expansion.
This indicates that for any $(\bm{u},v):\|\bm{u}\|_{2}^{2}+v^{2}=1$, we have 
$$\left|\sum_{i=1}^{n}\ell'(-y_{i}f(x_{i};\bm{\theta}^{*}))(-y_{i})(\bm{u}^{\top}\bm{h}^{(l-1)}(x_{i};\bm{\theta}^{*})+v)_{+}^{2}\right|<\lambda_{1}<\lambda_{0}$$
It follows from Lemma~\ref{lemma::multi-3} that if $\bm{v}_{1}^{*}\neq \bm{0},...,\bm{v}_{l-1}^{*}\neq \bm{0}$, then for any two different samples $x_{i}$ and $x_{j}$ in the dataset $\mathcal{D}$, we should have $$\bm{h}^{(l-1)}(x_{i};\bm{\theta}^{*})\neq \bm{h}^{(l-1)}(x_{j};\bm{\theta}^{*}).$$
Therefore, $$\bm{h}^{(l-1)}(x_{i};\bm{\theta}^*)\neq \bm{h}^{(l-1)}(x_{j};\bm{\theta}^*)$$ for any two different samples in the dataset.

Now we consider the following dataset under the parameters $\bm{\theta}^{*}$,
$$\tilde{\mathcal{D}}=\{(\bm{h}^{(l-1)}(x_{i};\bm{\theta}^{*}), y_{i})\}_{i=1}^{n}.$$
By Lemma~\ref{lemma::multi-3}, since the local minimum $\bm{\theta}^{*}$ satisfies $\|\bm{v}^{*}_{1}\|_{2}\ge1,...,\|\bm{v}^{*}_{l-1}\|_{2}\ge1$,  then in the dataset $\tilde{\mathcal{D}}$, $\bm{h}^{(l-1)}(x_{i})\neq \bm{h}^{(l-1)}(x_{j})$ for any $i\neq j$. 
Furthermore, by Claim~\ref{lemma::single-requ} in Appendix~\ref{appendix::single-requ}, it follows that for any given $\bm{v}_{1},...,\bm{v}$, there exists a single-layered ReQU network $p(x;\bm{\rho})=\sum_{i=1}^{n+1}\alpha_{j}(\bm{\omega}_{j}^{\top}x+\beta_{j})_{+}^{2}$ of size $n+1$ such that $\min_{i}y_{i}p\left(\bm{h}^{(l-1)}({x_{i};\bm{\theta}});\bm{\rho}^{}\right)>0$, where the vector $\bm{\rho}$ contains all parameters in the network $p$. 
Thus, let 
\begin{align*}
\bar{\lambda}&\triangleq \min_{\|\bm{v}_{1}\|_{2}\ge1,...,\|\bm{v}_{l-1}\|_{2}\ge1}\max_{\bm{\rho}:\|\bm{\rho}\|_{2}=1}\min_{i\in[n]}y_{i}p\left(\bm{h}^{(l-1)}({x_{i};\bm{\theta}});\bm{\rho}^{}\right).
\end{align*}
We note here that the output of the $(l-1)$-th layer $\bm{h}^{(l-1)}(x;\bm{\theta})$ is only dependent on parameters $\bm{v}_{1},...,\bm{v}_{l-1}$.
Let $(\bar{\bm{v}}_{1},...,\bar{\bm{v}}_{l-1},\bm{\rho}^{*})$ be a solver of the above minimization problem and let $\lambda_{0}=\varepsilon\bar{\lambda}$, then it easy to see that 
$$\bar{\lambda}=\min_{\|\bm{v}_{1}\|_{2}\ge1,...,\|\bm{v}_{l-1}\|_{2}\ge1}\min_{i\in[n]}y_{i}p\left(\bm{h}^{(l-1)}({x_{i};\bm{\theta}});\bm{\rho}^{*}\right)\le \min_{i\in[n]}y_{i}p\left(\bm{h}^{(l-1)}({x_{i};\bm{\theta}^{*}});\bm{\rho}^{*}\right)$$
and that
\begin{align*}
\bar{\lambda}\sum_{i=1}^{n}\ell'(-y_{i}f(x_{i};{\bm{\theta}}^{*}))&<\left|\sum_{i=1}^{n}\ell'(-y_{i}f(x_{i};\bm{\theta}^{*}))(-y_{i})p\left(\bm{h}^{(l-1)}({x_{i};\bm{\theta}^{*}});\bm{\rho}^{*}\right)\right|\\
&\le\sum_{j=1}^{m}\left|\sum_{i=1}^{n}\ell'(-y_{i}f(x_{i};{\bm{\theta}}^{*}))y_{i}\alpha_{j}^{*}({\bm{\omega}_{j}^{*}}^{\top}\bm{h}^{(l-1)}(x_{i};\bm{\theta}^{*})+\beta_{j}^{*})_{+}^{2}\right|\\
&<\sum_{j=1}^{m}(\|\bm{\omega}^{*}_{j}\|^{2}+{\beta_{j}^{*}}^{2})\lambda_{0}\le\lambda_{0}=\e\bar{\lambda}.
\end{align*}
Therefore, we have for each $i\in[n]$, $\ell'(-y_{i}f(x_{i};\bm{\theta}^{*}))<\e$ and this indicates $$y_{i}f(x_{i};\bm{\theta}^{*})>0$$ for $\forall i\in[n]$. 
\end{proof}


\section{Proof of Theorem~\ref{thm::multi}}\label{appendix::multi}
The sketch proof of Theorem~\ref{thm::multi} is as follows. We first prove that, at every critical point, there are only three cases: (1) all parameters in the last layer are zero and one of the first $(l-1)$ layers has a zero parameter vector, i.e., $\|\bm{a}^{*}\|_{2}=\|\bm{W}^{*}\|_{2}=0$ and  $\exists k\in[l-1]$ s.t. $\|\bm{v}_{k}^{*}\|_{2}=0$; (2) all parameters in the last layer are zero and all parameter vectors in the first $(l-1)$ layers are unit vectors, i.e., $\|\bm{a}^{*}\|_{2}=\|\bm{W}^{*}\|_{2}=0$ and $\|\bm{v}_{1}^{*}\|_{2}=...=\|\bm{v}_{l-1}^{*}\|_{2}=1$; (3) some parameters in the layer are non-zero and the 2-norm of all parameter vectors in the first $(l-1)$ layers is larger than one, i.e., $\|\bm{a}^{*}\|_{2}>0, \|\bm{W}^{*}\|_{2}>0$ and $\|\bm{v}_{1}^{*}\|_{2}=...=\|\bm{v}_{l-1}^{*}\|_{2}>1$. We next prove that the first and second case can never happen at a local minimum and also prove that any local minimum satisfying the third case has zero training error. 

For any critical point in the first case, we can easily check that it is not a local minimum by perturbing the parameters such that the regularizer part is decreasing while the classification loss part does not change. For the second and the third case, we first prove a lemma (Lemma~\ref{lemma::multi-1}) to show that if the width of the last layer is sufficiently large, then there is an inactive neuron in the last layer. This is trivially true for the second case but definitely non-trivial for the third case. We note that Lemma~\ref{lemma::multi-1} is analogous to Lemma~\ref{lemma::single-1} and thus the proof technique for Lemma~\ref{lemma::multi-1} is similar to the technique for proving Lemma~\ref{lemma::single-1}. The main technique for proving Lemma~\ref{lemma::multi-1} is showing that one of a series of matrices is non-singular when the width of the last layer is sufficiently large, which is given by Lemma~\ref{lemma::multi-2}. 

Next, we show by Lemma~\ref{lemma::multi-3} that if all parameter vectors in the first $(l-1)$ layers are non-zero, then for different inputs $x$ and $x'$, the outputs $\bm{h}^{(l-1)}(x)$ of the $(l-1)$-th layer are also different, i.e., $\bm{h}^{(l-1)}(x)=\bm{h}^{(l-1)}(x')$. Based on Lemma~\ref{lemma::multi-1} and Lemma~\ref{lemma::multi-3}, we present Lemma~\ref{lemma::multi-5} to show that the second case cannot happen at a local minimum and every local minimum satisfying the third case achieves zero training error and thus we complete the proof.  Now we present the formal proof of Theorem~\ref{thm::multi}.
\begin{proof}
(1) It is straightforward to see that the empirical loss  is coercive.

(2) First, we need to note that the loss $\lossc$ is not always differentiable with respect to the vectors $\bm{v}_{1},...,\bm{v}_{l-1}$, since the Leaky ReLU is not always differentiable. Assume $\bm{\theta}^{*}$ is a local minimum of the loss $L_n(\bm{\theta};\bm{\lambda})$. Given a local minimum $\bm{\theta}^{*}$, now we consider the  neural network with the parameters $\bm{\theta}(\bm{r})=(r_{l+1}\bm{a}^{*},r_{l}\bm{W}^{*},r_{l-1}\bm{v}^{*}_{l-1},...,r_{1}\bm{v}^{*}_{1})$, then the output of the neural network can be written as 
\begin{align*}
f(x;\bm{\theta}(\bm{r}))&=r_{l+1}{\bm{a}^{*}}^{\top}\text{ReQU}\left(r_{l}\bm{W}^{\top}\bm{\sigma}(r_{l-1}\bm{v}^{*}_{l-1}\ast\bm{\sigma}(...\ast\bm{\sigma}(r_{1}\bm{v}^{*}_{1}\ast x)))\right)\\
&=(r_{1}...r_{l})^{2}r_{l+1}f(x;\bm{\theta}^{*}),
\end{align*}
where $\bm{r}\in\mathbb{R}_+^{l+1}$.
Here we use the positive homogeneity of Leaky ReLU and ReQU: $\requ(rz)=r^2\requ(z)$ and $\sigma(rz)=r\sigma(z)$ for any $r\in\mathbb{R}_+$ and any $z\in\mathbb{R}$.
Given a local minimizer $\bm{\theta}^{*}$ of the empirical loss $\lossc$, we define a loss function $L(\bm{r};\bm{\theta}^{*},\bm{\lambda})$ as a function of the vector $\bm{r}=(r_{1},...,r_{l+1})$, 
\begin{align*}
&L(\bm{r};\bm{\theta}^{*},\bm{\lambda})\triangleq L_n(\bm{\theta}(r);\bm{\lambda})\\
&=\sum_{i=1}^{n}\ell(-y_{i}(r_{1}...r_{l})^{2}r_{l+1}f(x;\bm{\theta}^{*}))+\frac{1}{3}\sum_{j=1}^{m}\lambda_{j}\left[|r_{l+1}|^{3}|a_{j}^{*}|^{3}+2|r_{l}|^{3}(\|\bm{w}_{j}^{*}\|_{2}^{2}+{b_{j}^{*}}^{2})^{3/2}\right]\\
&\quad+\frac{\lambda_{c}}{2}\sum_{k=1}^{l-1}(r_{k}^{2}\|\bm{v}^{*}_{k}\|_{2}^{2}-1)^{2}
\end{align*}
Since $\bm{r}=\bm{1}_{l+1}$ is a local minimum of $L(\bm{r};\bm{\theta}^{*};\bm{\lambda})$ and the empirical loss is differentiable with respective to the vector $\bm{r}$ on $\mathbb{R}_{+}^{l+1}$, then we have 
$$\frac{\partial L(\bm{r};\bm{\theta}^{*};\bm{\lambda})}{\partial r_{1}}\Bigg|_{\bm{r}=\bm{1}_{l+1}}=...=\frac{\partial L(\bm{r};\bm{\theta}^{*};\bm{\lambda})}{\partial r_{L+1}}\Bigg|_{\bm{r}=\bm{1}_{l+1}}=0.$$
Since for $k=1,...,l-1$,
\begin{align*}
\frac{\partial L(\bm{r};\bm{\theta}^{*};\bm{\lambda})}{\partial r_{k}}\Bigg|_{\bm{r}=\bm{1}_{l+1}}&=2\sum_{i=1}^{n}\ell'(-y_{i}f(x_{i};\bm{\theta}^{*}))(-y_{i}f(x_{i};\bm{\theta}^{*}))+2\lambda_{c}(\|\bm{v}^{*}_{k}\|_{2}-1)\|\bm{v}^{*}_{k}\|^{2}_{2}\\
&=0,
\end{align*}
for $k=l$
\begin{align*}
\frac{\partial L(\bm{r};\bm{\theta}^{*};\bm{\lambda})}{\partial r_{l}}\Bigg|_{\bm{r}=\bm{1}_{l+1}}&=2\sum_{i=1}^{n}\ell'(-y_{i}f(x_{i};\bm{\theta}^{*}))(-y_{i}f(x_{i};\bm{\theta}^{*}))+2\sum_{j=1}^{m}\lambda_{j}(\|\bm{w}_{j}^{*}\|_{2}^{2}+{b_{j}^{*}}^{2})^{3/2}\\
&=0,
\end{align*}
and for $k=l+1$
\begin{align*}
\frac{\partial L(\bm{r};\bm{\theta}^{*};\bm{\lambda})}{\partial r_{l+1}}\Bigg|_{\bm{r}=\bm{1}_{l+1}}&=\sum_{i=1}^{n}\ell'(-y_{i}f(x_{i};\bm{\theta}^{*}))(-y_{i}f(x_{i};\bm{\theta}^{*}))+\sum_{j=1}^{m}\lambda_{j}|a_{j}^{*}|^{3/2}=0,
\end{align*}
then we have
\begin{align*}
\sum_{j=1}^{m}\lambda_{j}|a_{j}^{*}|^{3/2}&=\sum_{j=1}^{m}\lambda_{j}(\|\bm{w}_{j}^{*}\|_{2}^{2}+{b_{j}^{*}}^{2})^{3/2}\\
&=\lambda_{c}(\|\bm{v}^{*}_{1}\|_{2}-1)\|\bm{v}^{*}_{1}\|_{2}=...=\lambda_{c}(\|\bm{v}^{*}_{l-1}\|^{2}_{2}-1)\|\bm{v}^{*}_{l-1}\|^{2}_{2}
\end{align*}
Since $\lambda_{1},...\lambda_{m}$ and $\lambda_{c}$ are all positive numbers, then if $\bm{\theta}^{*}$ is a local minimum of the loss $\lossc$,  we should have the following three cases:

\textbf{Case 1}: $\|\bm{a}^{*}\|_{2}=\|\bm{W}^{*}\|_{2}=0$ and there exists $k\in[l-1]$ such that $\|\bm{v}_{k}^{*}\|_{2}=0$

\textbf{Case 2:} $\|\bm{a}^{*}\|_{2}=\|\bm{W}^{*}\|_{2}=0$ and $\|\bm{v}_{1}^{*}\|_{2}=...=\|\bm{v}_{l-1}^{*}\|_{2}=1$

\textbf{Case 3:} $\|\bm{a}^{*}\|_{2}>0, \|\bm{W}^{*}\|_{2}>0$ and $\|\bm{v}_{1}^{*}\|_{2}=...=\|\bm{v}_{l-1}^{*}\|_{2}>1$

Now we start from the first case. 

\textbf{Case 1:} Now we prove the following result. If there exists $k\in[l-1]$ such that $\|\bm{v}_{k}^{*}\|_{2}=0$, then $\bm{\theta}^{*}$ is a saddle point of the loss $\loss$. We assume that $k\in[l-1]$ and $\|\bm{v}_{k}^{*}\|_{2}=0$. Now we choose the following perturbed parameters $$\tilde{\bm{\theta}}=(\tilde{\bm{a}},\tilde{\bm{W}},\tilde{\bm{v}}_{1},...,\tilde{\bm{v}}_{k-1},\tilde{\bm{v}}_{k},\tilde{\bm{v}}_{k+1},...,\tilde{\bm{v}}_{l-1})=(\bm{a}^{*},\bm{W}^{*},\bm{v}^{*}_{1},...,\bm{v}_{k-1}^{*},\delta \bm{u},\bm{v}_{k}^{*},...,\bm{v}_{l-1}^{*}),$$
where $\bm{u}$ is an arbitrary unit vector $\|\bm{u}\|_{2}=1$. This means that  we only perturb the vector $\bm{v}_{k}$. Since $\|\tilde{\bm{a}}\|_{2}=0$, then it is easy to see that $f(x_{i};\tilde{\bm{\theta}})=0$ for all $i\in[n]$. By the definition of the local minimum, there exists a sufficiently small $\delta_{0}>0$, such that  for any $0<\delta<\delta_{0}$ and $\bm{u}:\|\bm{u}\|_{2}=1$, we have 
\begin{align*}
L_n(\tilde{\bm{\theta}};\bm{\lambda})&=\sum_{i=1}^{n}\ell(0)+\frac{\lambda_{c}}{2}\sum_{j\neq k}(\|\bm{v}_{j}^{*}\|_{2}^{2}-1)^{2}+\frac{\lambda_{c}}{2}(\delta^{2}-1)^{2}\\
&\ge \sum_{i=1}^{n}\ell(0)+\frac{\lambda_{c}}{2}\sum_{j\neq k}(\|\bm{v}_{j}^{*}\|_{2}^{2}-1)^{2}+\frac{\lambda_{c}}{2}=L_n(\bm{\theta}^{*};\bm{\lambda})
\end{align*}
This above inequality implies that there exists a sufficiently small $\delta_{0}>0$ such that for all $0<\delta<\delta_{0}$ we have $1-\delta^{2}>1$. This is clearly incorrect. Therefore, this indicates that if $\|\bm{a}^{*}\|_{2}=\|\bm{W}^{*}\|_{2}=0$ and there exists $k\in[l-1]$ such that $\|\bm{v}_{k}^{*}\|_{2}=0$, then $\bm{\theta}^{*}$ is not a local minimum. 

\textbf{Case 2:} From Lemma~\ref{lemma::multi-5}, it follows that $y_{i}f(x_{i};\bm{\theta}^{*})>0$ holds for all $i\in[n]$. However, since $\|\bm{a}\|_{2}^{*}=0$, then this leads to the contradiction.  This means that $\bm{\theta}^{*}$ satisfying $\|\bm{a}^{*}\|_{2}=\|\bm{W}^{*}\|_{2}=0$ and $\|\bm{v}_{1}^{*}\|_{2}=...=\|\bm{v}_{l-1}^{*}\|_{2}=1$ is not a local minimum. 

\textbf{Case 3:} From Lemma~\ref{lemma::multi-5}, it follows that $y_{i}f(x_{i};\bm{\theta}^{*})>0$ holds for all $i\in[n]$.
\end{proof}

\section{Claim~\ref{lemma::single-requ}}\label{appendix::single-requ}

We also need a simple claim that states
that the data can be interpolated by the neural network. 
\begin{claim}\label{lemma::single-requ}
Given a dataset $\mathcal{D}=\{(x_{i},y_{i})\}_{i=1}^{n}$ satisfying assumption \ref{assumption::dataset},  there exists a single-layered ReQU network $p$ of size $n+1$ and parameterized by $\bm{\rho}$ such that the network $p$ can correctly all samples in the dataset with a positive margin, i.e., $y_{i}p(x_{i};\bm{\rho})>0$ for all $i\in[n]$.
\end{claim}

\begin{proof}
Given a dataset $\mathcal{D}=\{(x_{i},y_{i})\in\mathbb{R}^{d}\times\{-1,1\}\}_{i=1}^{n}$ satisfying Assumption~\ref{assumption::dataset}, the set $$\mathcal{H}=\bigcup_{i\neq j}\{\bm{\omega}\in\mathbb{R}^{d}|\bm{\omega}^{\top}(x_{i}-x_{j})=0\}$$
is a zero measure set since $x_{i}-x_{j}\neq \bm{0}$ for any $i\neq j$. Therefore, there always exists a $\bm{\omega}_{0}\in\mathbb{R}^{d}$ such that 
$$\bm{\omega}_{0}^{\top}(x_{i}-x_{j})\neq 0,\quad \text{for }i\neq j.$$
Now we define $z_{i}=\bm{\omega}_{0}^{\top}x_{i}$, then it is easy to see that $z_{i}\neq z_{j}$ for any $i\neq j$. Without loss of generality, we assume that $z_{1}<z_{2}<...<z_{n}$. First, we choose a number $z_{0}$ smaller than $z_{1}$, i.e., $z_{0}<z_{1}$. 
 Now, we define a series of functions $q_{1},...,q_{n}$ as follows:
\begin{align*}
q_{1}(z)&=y_{1}(z-z_{0})_{+}^{2}\\
q_{k+1}(z)&=q_{k}(z)+\frac{2y_{k+1}(|q_{k}(z_{k+1})|+1)}{(z_{k+1}-z_{k})^{2}}(z-z_{k})_{+}^{2},\quad k=1,...,n-1
\end{align*}
It is easy to check that the function $q_{n}$ has the following property
$$y_{k}q_{n}(z_{k})=y_{k}q_{k}(z_{k})>0, \quad k=1,...,n.$$
This can be easily proved by induction. For the base case $k=1$, $q_{1}(z_{1})=y_{i}(z_{1}-z_{0})^{2}>0$. For the case $k$, assume that $q_{k}(z_{i})=q_{i}(z_{i})$ and $y_{i}q_{i}(z_{i})>0$ for $i=1,...,k$. Then for case $k+1$, we have 
$$q_{k+1}(z)=q_{k}(z)+\frac{2y_{k+1}(|q_{k}(z_{k+1})|+1)}{(z_{k+1}-z_{k})^{2}}(z-z_{k})_{+}^{2}.$$
Since $z_{i}\le z_{k}$ for any $i=1,...,k$, then for $i=1,...,k$
$$y_{i}q_{k+1}(z_{i})=y_{i}q_{k}(z_{i})=y_{i}q_{i}(z_{i})>0$$
and $$y_{k+1}q_{k+1}(z_{k+1})=y_{k+1}q_{k}(z_{k+1})+2|q_{k}(z_{k+1})|+2\ge |q_{i}(z_{i})|+2>0.$$
Therefore, for case $k+1$, we have 
$$y_{i}q_{k+1}(z_{i})=y_{i}q_{i}(z_{i})>0,\quad i=1,...,k+1.$$
By induction, we can prove $$y_{k}q_{n}(z_{k})=y_{k}q_{k}(z_{k})>0, \quad k=1,...,n.$$
It is easy to see that $q_{n}$ is a single layered ReQU network of size $n$, thus we finish the proof. 
\end{proof}


%


\bibliography{nips_2018}
\bibliographystyle{unsrt}

\end{document}